\ifdefined\pdfoutput
\pdfoutput=1
\fi

\documentclass[11pt]{scrartcl}
\usepackage{array}
\usepackage{graphicx,subcaption,multirow,listings,float}
\usepackage[section]{placeins}
\usepackage{amsmath,amsthm,mathrsfs,amssymb}
\usepackage{bbm,bm,amsfonts}
\usepackage[mathscr]{eucal}
\usepackage{diagbox}
\usepackage{xcolor}
\usepackage{manyfoot,booktabs,textcomp}
\usepackage{algorithm,algorithmicx,algpseudocode}
\usepackage{authblk}
\usepackage{todonotes}
\usepackage[T1]{fontenc}
\usepackage[utf8]{inputenc}
\usepackage[a4paper,left=2.5cm,right=2.5cm]{geometry}
\usepackage{hyperref}
\usepackage[capitalize]{cleveref}
\DeclareOldFontCommand{\rm}{\normalfont\rmfamily}{\mathrm}
\ifdefined\DeclareUnicodeCharacter
\DeclareUnicodeCharacter{2011}{-}
\DeclareUnicodeCharacter{2013}{--}
\fi
\hypersetup{
  colorlinks=true,
  linkcolor=blue,
  urlcolor=blue,
  citecolor=blue
}

\newcommand{\tkappa}{\tilde{\kappa}}
\newcommand{\tgamma}{\tilde{\gamma}}
\newcommand{\teta}{\tilde{\eta}}
\newcommand{\bE}{\mathbb E}
\newcommand{\bS}{\mathbb S}
\newcommand{\bR}{\mathbb R}
\newcommand{\bZ}{\mathbb Z}
\newcommand{\bM}{\mathbb M}

\newcommand{\fq}{\mathbf q}
\newcommand{\fs}{\mathbf s}
\newcommand{\fx}{\mathbf x}
\newcommand{\fV}{\mathbf V}
\newcommand{\fp}{\mathbf p}

\newcommand{\kp}{\mathfrak p}

\newcommand{\kB}{\mathfrak B}

\newcommand{\kH}{\mathfrak H}
\newcommand{\kq}{\mathfrak q}

\newcommand{\cU}{\mathcal U}

\newcommand{\cG}{\mathcal G}
\newcommand{\cH}{\mathcal H}

\newcommand{\cC}{\mathcal C}
\newcommand{\cF}{\mathcal F}
\newcommand{\cB}{\mathcal B}
\newcommand{\cT}{\mathcal T}
\newcommand{\cO}{\mathcal O}

\newcommand{\kC}{\mathfrak{C}}
\newcommand{\rn}{\mathbf{n}}

\newcommand{\dfp}{\dot{\mathbf{p}}}
\newcommand{\dfe}{\dot{\mathbf e}}

\newcommand{\dtheta}{\dot\theta}
\newcommand{\htheta}{\hat\theta}
\newcommand{\hfp}{\hat{\mathbf{p}}}
\newcommand{\hkp}{\hat{\mathfrak{p}}}

\DeclareMathOperator\sign{sign}

\DeclareMathOperator\Lip{Lip}

\DeclareMathOperator\Length{Length}

\theoremstyle{plain}
\newtheorem{theorem}{Theorem}
\newtheorem{proposition}[theorem]{Proposition}%

\theoremstyle{definition}
\theoremstyle{remark}

\newcommand{\keywords}[1]{\par\noindent\textbf{Keywords:} #1}

\begin{document}

\title{Computing Smooth Geodesics under Two-Sided Curvature Bounds with Applications to Robotics and Image Analysis}

\author[1]{Da Chen}
\author[2]{Zhenjiang Li}
\author[3]{Jean-Marie Mirebeau}
\author[4]{Xuecheng Tai}
\author[5]{Jinglin Zhang}
\author[6]{Wei Zhang}
\author[1]{Laurent D. Cohen}
\affil[1]{CEREMADE, University Paris Dauphine, University-PSL, CNRS, UMR 7534, Paris, 75775, France}
\affil[2]{Department of Radiation Oncology, Shandong Cancer Hospital and Institute, Shandong First Medical University, Shandong Academy of Medical Sciences, Jinan, 250117, China}
\affil[3]{Department of Mathematics, Centre Borelli, ENS Paris-Saclay, CNRS, University Paris-Saclay, Gif-sur-Yvette, 91190, France}
\affil[4]{Norce, Nyg{\aa}rdsgaten 112, Bergen, 5008, Norway}
\affil[5]{School of Control Science and Engineering, Shandong University, Jinan, 610101, China}

\renewcommand\Authfont{\normalsize}
\renewcommand\Affilfont{\small}
\setlength{\affilsep}{0.35em}
\date{}
\maketitle

\begin{abstract}
Curvature of planar curves serves as a key regularization term for computing second-order minimal paths, due to its tight relevance to desirable geometric properties such as smoothness, rigidity, and elasticity. In this paper, we tackle a more challenging problem in  computational physics and geometry problem: tracking minimal paths whose curvature is constrained by arbitrary upper and lower bounds. For that purpose, we propose  a new curvature-bounded geodesic model, developed under the Hamilton-Jacobi-Bellman (HJB) partial differential equation (PDE) framework. It provides strong geometric control over minimal paths by enforcing curvature range constraints, whose paths are smooth and of bounded curvature limitation.
We also present a discretization scheme for the Hamiltonian and the HJB PDE incorporating curvature bounds, allowing efficient solver for estimating numerical solutions to the model. Finally, we illustrate the capability of the proposed curvature-bounded geodesic model in applications of  robot path planning and curvilinear  structures  tracking  from images. Numerical experiments demonstrate that the proposed curvature-bounded geodesic model serves as a powerful and robust tool for finding satisfactory paths.
\end{abstract}




\section{Introduction}
\label{sec_intro}
In recent years, the theoretical foundations and numerical methods for computing minimal paths in a connected domain have become the core to a wide range of scientific problems posed in applied mathematics, engineering and science~\cite{peyre2010geodesic}, due to their strong ability in accommodating mixed geometric, perceptual  and physical priors. The computation of  minimal paths bridges several research areas, including   artificial intelligence, image understanding and PDE numerical analysis, thus exhibiting its fundamental and important role in computational sciences.

Early approaches focus on the computation of minimal paths whose weighted curve length is dependent on the local path position and tangent, which are categorized as first-order geometric features~\cite{cohen1997global,benmansour2011tubular,chen2024region}. Despite successful applications in many research areas, these first-order minimal path models have difficulties in taking into account the path curvature and the associated geometric shape priors~\cite{chen2023geodesic,duits2018optimal,mirebeau2018fast,mashtakov2023time,chen2023computing}. Recent efforts are devoted to minimizing bending energy functionals featuring  different second-order curvature penalizations. Using path curvature as regularization leads to strong flexibility in embedding various geometry priors into the computation of minimal paths, and has succeeded in a variety of practical applications such as image analysis~\cite{mashtakov2017tracking,bekkers2015pde,deng2019new,liu2021color} and robotics motion planning~\cite{kimmel1998multivalued,mirebeau2017automatic,kimmel2001optimal}.

The connection between the minimization of a path energy, defined by an appropriate metric, and the geometric control theory paves the way for computing globally optimal paths between prescribed endpoints~\cite{mirebeau2018fast}.
The foundation of this research line is established over the framework of the first-order static HJB PDE, whose viscosity solution corresponds to the minimum of the path energy and can be numerically approximated via the seminal approach  named Fast-Marching method~\cite{sethian1996fast,sethian1999fast} and its elegant Finslerian variants featuring asymmetry property~\cite{mirebeau2018fast,mirebeau2019hamiltonian,mirebeau2014anisotropic,mirebeau2023massively,sethian2003ordered,sethian2001ordered}. The classical Euler-Mumford Elastica problem, first investigated in physics to describe the deformation of an elastic thin rod, is a typical instance for computing second-order curvature-penalized minimal paths,  whose bending energy involves an integral of squared path curvature. Following the pioneering approach in~\cite{mumford1994elastica}, the Euler-Mumford Elastica model has been developed as an efficient mathematical tool in the fields of computer vision and medical imaging~\cite{chen2017global,tai2011fast,ben2014tangent}. In its basic formulation, the Euler-Mumford Elastica model usually favors smooth minimal paths with homogeneously slow-varying tangents, which, unfortunately, is not always suitable in the realistic applications. 
Indeed, in the application domain of interest, minimal paths are expected to simultaneously maintain the rigid property of curves, and to fit the centerlines of curvilinear structures featuring segments whose curvature is strong and varies quickly. For that purpose, the curvature prior Elastica model~\cite{chen2023computing} was introduced as a generalization of the classical approach~\cite{chen2017global,mirebeau2018fast} to penalize a curvature drift term that measures the deviation of the path curvature from a prescribed feature map,  providing an avenue to take advantage of the bending features of curvilinear structures. However, in its basic formulation, the curvature drift term in essence imposes a soft constraint on the associated minimal paths, whose path curvature may not always agree with the prescribed values, especially when the geometric objects to process have challenging appearance.  This is also the case for the variant of the Reeds-Shepp forward model~\cite{van2024geodesic}, even though the two models~\cite{chen2023computing,van2024geodesic} are distinct from each other in the use of curvature priors and also in the definition of bending energy functionals.

\begin{figure*}[!htbp]
\centering
\includegraphics[width=\linewidth]{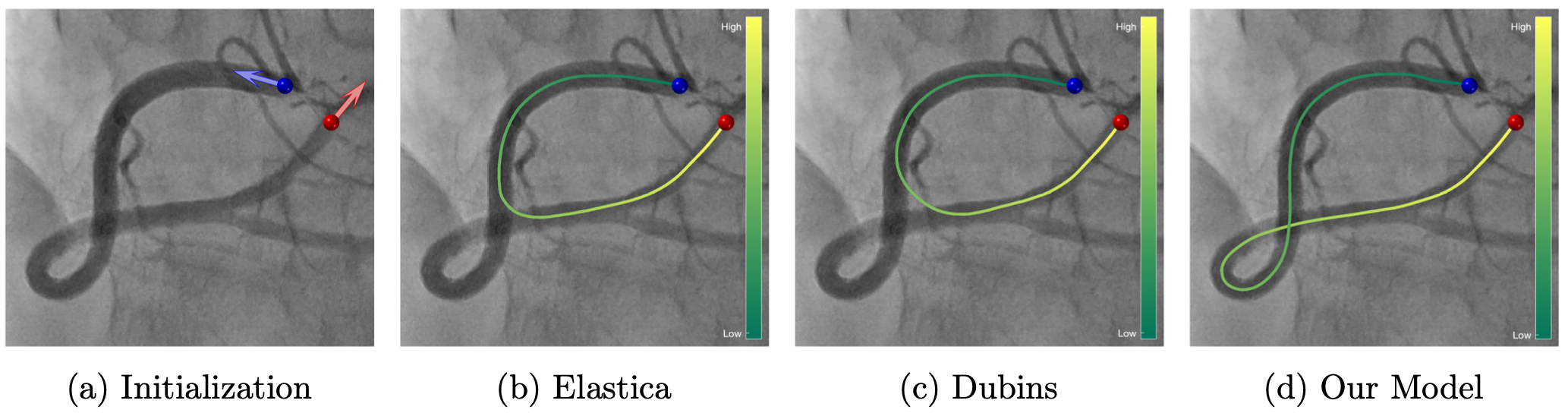}
\caption{Illustration for the advantages of the proposed curvature-bounded geodesic model comparing to state-of-the-art second-order curvature-penalized models. (\textbf{a}) The initialization information over an x-ray heart image, which contains a self-crossing pattern.  The blue and red dots represent the source point and the end point with arrows indicating the path tangents at those points. (\textbf{b})-(\textbf{d}) The lines indicate the physical projection curves of the minimal paths which are respectively derived from the classical Euler-Mumford Elastica model, the Dubins car model and the introduced curvature-constrained Elastica model. The varying color of the lines represent the euclidean curve length of those paths.}
\label{fig_CompXRay}
\end{figure*}

The Dubins model~\cite{mirebeau2018fast} is designed for computing minimal paths with strong curvature penalization, where the magnitude of path curvature is constrained~\cite{boissonnat1994shortest}. However, unlike the Euler-Mumford Elastica model, a key limitation of the Dubins model is its lack of  smoothness property in the yielded minimal paths---an  essential property in many practical  applications. In this article, we propose a variant of the classical Euler-Mumford Elastica model~\cite{chen2017global}, referred to  as the \emph{curvature-bounded model}, which is able impose \emph{arbitrary} lower and upper bounds to directly control the path curvature of the planar component of the computed  minimal paths. 
For that purpose, we define a new Hamiltonian, we study the corresponding control sets, which implicitly embed the limitation of path curvature, and we design numerical schemes for solving the corresponding HJB PDE using the geometric optimization tool of Voronoi's first reduction~\cite{mirebeau2018fast,mirebeau2019riemannian}. This allows us to interpret our curvature-bounded model as a minimum arrival time problem or a minimal action problem by means of the optimal control theory, and also yields an efficient way to find globally optimal curves of bounded curvature between two endpoints with given tangent directions at both points. In particular,  the model we propose here provides a hard constraint to limit the path curvature, thus serving as an alternative way to efficiently implement the curvature prior enhancement for finding satisfactory minimal paths. 
This gives our model a clear advantage over existing minimal path methods, particularly in challenging scenarios, highlighting that the introduction of the curvature-bounded model is both significant and nontrivial.
In Fig.~\ref{fig_CompXRay}, we illustrate the qualitative comparison results of the Euler-Mumford Elastica model~\cite{chen2017global}, the Dubins car model~\cite{mirebeau2018fast} and the introduced bounded elastica model. Those numerical experiments demonstrated in Fig.~\ref{fig_CompXRay} are conducted over an X-ray heart image. The objective is to compute a minimal path such that its physical projection curve can accurately depict the centerline of a blood vessel featuring a self-crossing structure. The initialization for the tested models is illustrated in Fig.~\ref{fig_CompXRay}a, where the blue and red dots respectively denote the source point and the target point with arrows pointing the tangent directions at those points. From this experiment, one  can see that both of the Euler-Mumford elastica model and the Dubins car model, whose minimal paths are respectively shown in Figs.~\ref{fig_CompXRay}b and~\ref{fig_CompXRay}c, fail to track the centerlines of the self-crossing structures. In contrast, the introduced bounded Elastica model blends the benefits of the smoothness property and the enhancement from the arbitrary curvature bounds, thus is able to successfully extract the self-crossing structure, as illustrated in Fig.~\ref{fig_CompXRay}d.

The remainder of this manuscript is organized as follows. We first summarize the introduced bounded Elastica model and the mathematical tools used. Following that we present the main contribution of this work: 
\begin{itemize}
	\item The theoretical results involving the expression of the Hamiltonian of the bounded Elastica model and the corresponding control sets.
	\item The practical applications involving the curvilinear structure tracking and robot motion planning. Eventually, the numerical scheme for tracking accurate minimal paths from a Cartesian grid is presented in the Method section.
\end{itemize}

\section{Results}
\label{sec_res}

\begin{figure*}[!htbp]
\centering
\includegraphics[width=0.8\linewidth]{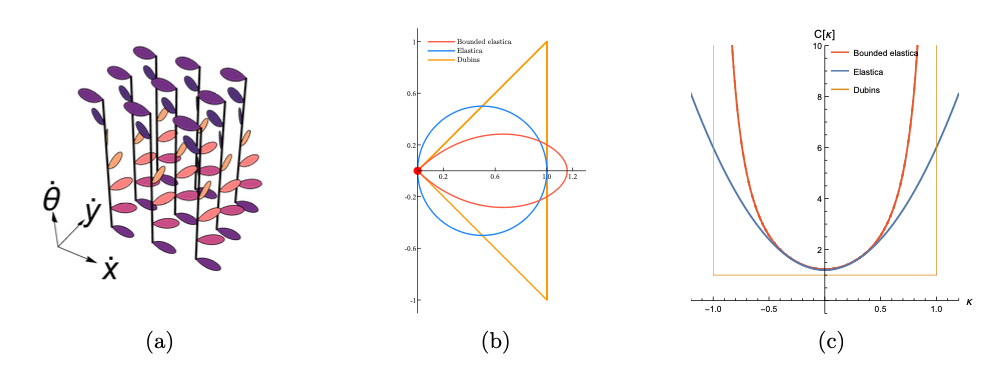}
\caption{
(\textbf{a}) Tissot indicatrix of the proposed bounded Elastica model, i.e.\ illustration of the control set $\cB(\fp)$ of admissible velocities at a collection of points $\fp = (x,y,\theta) \in \bM := \Omega \times \bS^1$ of the domain. (\textbf{b}) Comparison of the control sets of the proposed bounded Elastica model, with the Euler-Mumford Elastica (with adjusted parameters) and Dubins models. (\textbf{c}) Comparison of the curvature penalties. The proposed bounded Elastica model has a smooth and convex curvature penalty, with vertical asymptotes as $\kappa \to 1^-$ and $\kappa \to -1^+$. The curvature penalty of the Euler-Mumford Elastica model has the form $a+b\kappa^2$, for user-defined parameters $a,b>0$. The curvature penalty of the Dubins model is constant if $|\kappa|<1$, and $+\infty$ otherwise.
}
\label{fig_ControlsJMM}
\end{figure*}

\subsection*{Model Overview}
The bounded Elastica model introduced in this paper is designed in such way that the optimal curves are smooth and that their curvature is constrained within arbitrary upper and lower bounds. Similarly to other state-of-the-art curvature-penalized minimal path models~\cite{chen2017global,mirebeau2018fast,duits2018optimal,chen2023computing} based on the HJB PDE framework, the state space underlying the curvature-bounded model is a three-dimensional orientation lifted space, defined as the product $\bM:=\Omega\times\bS^1$ of a two dimensional open-bounded \emph{physical} domain $\Omega\subset\bR^2$, with the \emph{angular} domain $\bS^1:=\bR/2\pi\bZ$ (equivalently, $\bS^1 = [0,2\pi[$ with periodic boundary conditions).
An arbitrary point $\fp=(\kp,\theta)\in\bM$ thus involves two components: $\kp\in\Omega$ the physical position, and $\theta\in\bS^1$ the angular coordinate.
Consider a parametrized curve $\gamma:[0,1]\to\Omega$ in the physical domain $\Omega$, which is twice continuously differentiable and \emph{regular} in the sense that its velocity is non-vanishing. This curve can be lifted to $\wp=(\gamma,\eta):[0,1]\to\bM$  where the function  $\eta:[0,1]\to\bS^1$ characterizes the turning angles of  the \emph{physical projection curve} $\gamma$, namely $\gamma^\prime(t)=\|\gamma^\prime(t)\|\rn(\eta(t)),\,\forall t\in[0,1]$, where $\rn(\theta)=(\cos\theta,\sin\theta)$ is the unit vector associated to an angle $\theta\in\bS^1$. The curvature $\kappa:[0,1]\to\bR$ of the physical projection curve $\gamma$ equals the ratio 
\begin{equation*}
\kappa=\eta^\prime/\|\gamma^\prime\|.
\end{equation*} 
The curvature $\kappa$ of the physical projection curve $\gamma$ is thus expressed in terms of the tangent vector $\wp^\prime = (\gamma^\prime,\eta^\prime)$ to the \emph{orientation-lifted path}.  

The proposed curvature-bounded geodesic model can be formulated as a minimum arrival time optimal control problem in the configuration space $\bM$, in which the velocity of an orientation-lifted path is constrained within suitable control sets $\cB$. 
Each control set $\cB(\fp)$ at a point $\fp\in\bM$ is a \emph{convex} and \emph{compact} subset of the tangent space $\bE:=\bR^2\times\bR$ containing the origin $\mathbf{0}=(0,0,0)$, and depending continuously on $\fp$ w.r.t.\ the Haussdorff distance; the control sets corresponding to our specific model are illustrated in Fig.~\ref{fig_ControlsJMM} and will be explicitly constructed in the next section.
The minimum arrival time problem from a source point $\fs\in\bM$ to a target point $\fp\in\bM$ is defined as 
\begin{align}
\label{eq_MinimalTime}
\cT_\cB(\fs,\fp):=\inf\Bigl\{T>0\,|\,&\exists\wp\in\Lip([0,1],\bM),\nonumber\\
&\wp(0)=\fs,\,\wp(1)=\fp,\nonumber\\
&T^{-1}\wp^\prime(t)\in\cB(\wp(t)),\,\forall t\in[0,1]\Bigr\}
\end{align}
where $\Lip([0,1],\bM)$ denotes the collection of all Lipschitz continuous curves $\wp:[0,1]\to\bM$.  
In~\cref{eq_MinimalTime}, the condition $T^{-1}\wp^\prime\in\cB(\wp)$ defines the  $T\cB$-admissibility of a curve $\wp\in\Lip([0,1],\bM)$. 
If the infimum $T:=\cT_{\cB}(\fs,\fp)$ is finite, then it is attained, i.e.\ there exists a $T\cB$-admissible path $\cG_{\fs,\fp}$ from $\fs$ to $\fp$ referred to as an optimal curve or minimal path, see the literature~\cite{bardi1997Optimal} or~\cite[Appendix B]{chen2017global} on optimal control.

The tangents $T^{-1} \cG_{\fs,\fp}^\prime\in\cB(\cG_{\fs,\fp})$ to the minimal path $\cG_{\fs,\fp}=(\tgamma_{\fs,\fp},\teta_{\fs,\fp})$ contain all the information for computing the path curvature $\tkappa$ of its physical projection curve $\tgamma_{\fs,\fp}$. By designing proper control sets $\cB$, in the following, we are able to constrain the curvature $\tkappa$: 
\begin{equation}
\label{eq_curvatureBounds}
\Im_{\rm min}\bigl(\tgamma_{\fs,\fp}(t),\teta_{\fs,\fp}(t)\bigr)\leq \tkappa(t)\leq \Im_{\rm max}\bigl(\tgamma_{\fs,\fp}(t),\teta_{\fs,\fp}(t)\bigr)
\end{equation}
for any $t\in[0,1]$, where $\Im_{\rm min}$, $\Im_{\rm max}:\bM\to\bR$ are two user-defined scalar-valued continuous functions setting the lower and upper curvature bounds.

We discuss below some of the qualitative properties of the control sets of the proposed curvature-bounded model, anticipating on their precise mathematical definition which is given in the next section, and which involves four parameters here arbitrarily fixed to 
\begin{equation}
\label{eq_default_params}
	\cC=\xi=1 
	\quad \text{and} \quad 
	\psi_{\max} = -\psi_{\min} = \pi/4.
\end{equation}
The \emph{Tissot indicatrix} of the curvature-bounded model, illustrated in \cref{fig_ControlsJMM}a, shows the control sets $\cB(x,y,\theta)$ attached to a collection of regularly spaced points, i.e. the physical position $\kp = (x,y)\in \Omega$ and the angular coordinate $\theta\in \bS^1$ lie on a grid. Each of these control sets is a flat convex shape contained in a two-dimensional half plane $\bE_\theta$ of the tangent space $\bE := \bR^2 \times \bR$:
\begin{equation}
\label{eq_HalfPlane}
	\bE_\theta := \{ (\dot \nu \rn(\theta), \dot \theta) \mid \dot \nu \geq 0,\ \dot \theta \in \bR\},
\end{equation} 
where we recall that $\rn(\theta) := (\cos\theta,\sin\theta)$.
Our minimal path model thus describes a non-holonomic vehicle whose physical velocity $\dot{\nu}\rn(\theta)$ is always positively collinear with the heading direction $\rn(\theta)$, and is thus never directed backwards or sideways. 
Since $\cB(\fp) \subset \bE_\theta$, there is no loss of information in considering the following sliced control sets
\begin{equation}
\label{eq_BECS2D}
\tilde\cB(\fp)=\left\{(\dot\nu,\dtheta)\in\bR^2~|~\dot\nu\geq0,(\dot\nu\rn(\theta),\dtheta)\in\cB(\fp)\right\},
\end{equation}
as shown  in~\cref{fig_ControlsJMM}b, which is independent of $\fp$ under the assumption in~\cref{eq_default_params}.
Similarly to the Euler-Mumford Elastica model, the control sets of the proposed curvature-bounded model have a smooth (except at the origin)
 and strictly convex boundary, leading to smooth minimal paths; in contrast, the triangle-shaped control set of the Dubins model leads to piecewise smooth paths, which are known to be concatenations of straight segments and circular arcs. 
Similarly to the Dubins model, the boundary of the control sets of the curvature-bounded model has an angle at the origin, leading to upper and lower bounds on the admissible path curvature, see~\cref{prop_BoundedCurvature} for a detailed argument; in contrast, the Euler-Mumford Elastica model admits a vertical tangent at the origin, indicating that the curvature is not bounded a priori. 

Minimal paths for the proposed curvature-bounded model, the Euler-Mumford Elastica model and the Dubins model, can be investigated in the curvature penalization framework established in~\cite{mirebeau2018fast}.
A path $\wp$ from $\fp = (\kp,\theta)$ to $\fq = (\kq,\phi)$ is optimal iff the physical projection curve $\gamma : [0,L] \to \bR^2$, parametrized at unit Euclidean speed and whose curvature is denoted $\kappa : [0,L] \to \bR$, minimizes the second-order energy functional
\begin{equation*}
\int_0^L  \kC(\kappa(l)) dl,\ \text{subject to}\
\begin{cases}
\gamma(0) = \fp,\ \gamma^\prime(0) = \rn(\theta),\\
\gamma(L)=\fq,\ \gamma^\prime(L)=\rn(\phi).
\end{cases}
\end{equation*}
The curvature dependent cost function $\kC : \bR \to ]0,\infty]$ is related to the control sets via the equality
\begin{equation}
\label{eq_CurvCost}
\kC(\dot \theta/\dot \nu) = 1/\dot\nu \quad \text{for each~} (\dot \nu, \dot \theta) \in \partial \tilde \cB(\fp),
\end{equation}
with $\dot\nu>0$. 
The curvature cost function $\kC$ associated to the Euler-Mumford Elastica model is a parabola, whereas for the Dubins model it equals $1$ on the interval $[-1,1]$ and $+\infty$ elsewhere. Unfortunately, the curvature cost function $\kC$ associated with the curvature-bounded model does not have a closed form algebraic expression to our knowledge, but we can nevertheless compute it numerically using~\cref{eq_CurvCost}, see~\cref{fig_ControlsJMM}c. Similarly to the Euler-Mumford Elastica  model, it is smooth and approximately parabolic close to the origin; similarly to the Dubins model it equals $+\infty$ outside the interval $[-1,1]$.


Finally, let us mention that a much wider range of control set shapes and minimal path behaviors can be achieved by lifting the constraint of \cref{eq_curvatureBounds}, and letting the parameters $\cC$, $\xi$, $\psi_{\max}$, $\psi_{\min}$ depend on the current point $\fp$. One may for instance favor paths in going through some regions by decreasing the cost $\cC$ there, one may impose specific upper or lower curvature bounds via $\psi_{\max}$ and $\psi_{\min}$, and one may alter the curvature cost function through $\xi$. 

\subsection*{Hamiltonian and control sets of the proposed curvature-bounded model}
The control sets $\cB$ can be equivalently described in terms of a Hamiltonian $\cH:\bM\times\bE^*\to[0,\infty[$, where the co-tangent space $\bE^*$ is the dual to the tangent vector space $\bE = \bR^2 \times \bR$.
On the one hand, one can define
\begin{equation}
\label{eq_HamAsSup}
	\sqrt{2\cH(\fp,\hfp)} = \max\big\{ \langle\hat \fp,\dfp\rangle \mid \dot \fp \in \cB(\fp)\big\},
\end{equation}
where $\langle\dfp,\hfp \rangle$ denotes the standard Euclidean scalar product of $\dfp$ and $\hfp$.
The Hamiltonian $\cH(\fp,\hfp)$ is a continuous function of point $\fp \in \bM$ and co-vector $\hat \fp \in \bE^*$,  convex and positively two-homogeneous w.r.t. the second variable $\hat \fp$. Conversely, the control sets can be recovered as follows from a Hamiltonian obeying those properties:
\begin{equation}
\label{eq_BEControlSet}
\cB(\fp):=\overline{\mathrm{co}}\left\{\frac{\partial}{\partial\hfp}\sqrt{2\cH(\fp,\hfp)}\in\bE,\,\forall \hfp\in\bE^*\backslash\{\mathbf{0}\} \right\},
\end{equation}
where $\overline{\mathrm{co}}(A)$ denotes the closed convex envelope of a set $A \subset \bE$, and the points $\hat \fp$ where $\sqrt{2 \cH(\fp,\cdot)}$ is not differentiable are implicitly omitted from Eq.~\cref{eq_BEControlSet}. 
Therefore, instead of giving the explicit expression of the control sets $\cB$, we alternatively focus on the design of a Hamiltonian $\cH$ that enforces the desired curvature constraints, whose construction constitutes the main theoretical originality of this work. 

Let $\cC : \bM \to ]0,\infty[$ be a positive cost function characterizing the geometric features of the target curvilinear structures and let us define by $\Re(\theta,\psi)\in\bE$  a control vector with respect to a pair of angles $(\theta,\,\psi)$, reading as
\begin{equation}
\label{eq_ControlVector}
\Re(\theta,\psi)=\left(\cos(\psi)\rn(\theta),\xi^{-1}\sin(\psi)\right),
\end{equation}
where $\xi>0$ is a constant.
For any point $\fp=(\kp,\theta)\in\bM$ and any co-vector $\hfp=(\hkp,\htheta)\in\bE^*$,  we define the Hamiltonian $\cH$ of the proposed curvature-bounded model as follows
\begin{equation}
\label{eq_BEHamiltonian}
\cH(\fp,\hfp):=\frac{3}{8}\cC(\fp)^{-2}\,\int_{\psi_{\rm min}(\fp)}^{\psi_{\rm max}(\fp)}\,\langle\hfp, \Re(\theta,\psi)\rangle_+^2\,w_\fp(\psi)\,d\psi
\end{equation}
where $\langle\hfp,\dfp\rangle_+:=\max\{0,\langle\hfp,\dfp\rangle\}$ is the positive part of the standard Euclidean scalar product, and  $\psi_{\rm min},\,\psi_{\rm max}:\bM\to[-\pi/2,\pi/2]$ are scalar-valued functions computed from the curvature bounds $\Im_{\rm min} < \Im_{\rm max}$ in such way that 
\begin{align*}
&\psi_{\rm min}(\fp):=\arctan\bigl(\xi\Im_{\rm min}(\fp)\bigr),\\
&\psi_{\rm max}(\fp):=\arctan\bigl(\xi\Im_{\rm max}(\fp)\bigr).
\end{align*}
In addition, the non-negative weight $w_\fp(\psi)\geq 0$ in the integral is defined for each angle $\psi \in [\psi_{\min}(\fp),\psi_{\max}(\fp)]$ as
\begin{equation}
\label{eq_VarFejWeights}
w_\fp(\psi):=\frac{1}{\lambda_\fp}\cos\left(\frac{\psi-\Psi_\fp}{\lambda_\fp}\right)
\end{equation}
where the functions $\Psi_\fp$ and $\lambda_\fp$ are respectively defined as
\begin{align*}
&\Psi_\fp:=\frac{1}{2}\bigl(\psi_{\rm max}(\fp)+\psi_{\rm min}(\fp)\bigr), \\
&\lambda_\fp:=\frac{1}{\pi}\bigl(\psi_{\rm max}(\fp)-\psi_{\rm min}(\fp)\bigr).
\end{align*}
By those definitions we note that $\Psi_\fp$ is the average  of the values $\psi_{\rm max}(\fp)$ and $\psi_{\rm min}(\fp)$, and $\lambda_\fp$ is a scaling factor. 

The definition of the Hamiltonian in~\cref{eq_BEHamiltonian} via an integral is not very common in applications \cite{bardi1997Optimal,mirebeau2018fast}.
Typically, one either provides an explicit algebraic expression or  employs a supremum form analogous to that in~\cref{eq_HamAsSup}. The integral form~\cref{eq_BEHamiltonian} here is motivated by (i) its simple and accurate numerical implementation using Fejer quadrature as described in the Methods section, and (ii) its connection to the Euler-Mumford Elastica model (see~\cref{eq_ElasticaHam} below).

\begin{figure*}[!htbp]
\centering
\includegraphics[width=0.9\linewidth]{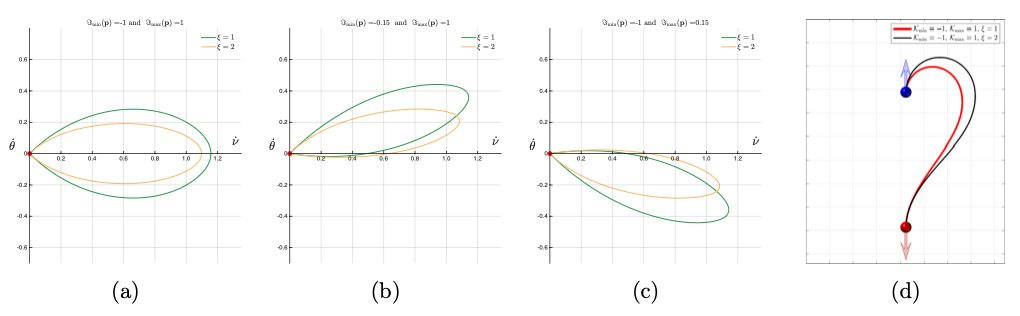}
\caption{Illustration for the plots of the boundaries $\partial\tilde{\cB}(\fp)$ of the proposed curvature-bounded model and the corresponding minimal paths. (\textbf{a})-(\textbf{c}): visualization for the boundaries $\partial\tilde{\cB}(\fp)$ with respect to different values of the parameter $\xi$.  The values of the curvature bounds $\Im_{\rm min}$ and $\Im_{\rm max}$ are shown at the top of each column. The red dots indicate the origin $(0,0)$ of the set $\tilde{\cB}(\fp)$. (\textbf{d}): the lines represent the physical  projections of orientation-lifted minimal paths with bounded curvature. The blue and red dots are respectively the source point and endpoint with arrows indicating the tangent directions assigned to the these points.}
\label{fig_ControlSets_Xi}
\end{figure*}

Differentiating the Hamiltonian expressed in~\cref{eq_BEHamiltonian} w.r.t. the co-vector $\hat \fp \in \bE^*$, we obtain at any point $\fp = (\kp,\theta) \in \bM$
\begin{equation*}
	\frac{\partial\cH}{\partial \hat \fp}(\fp,\hat \fp) = \frac 3 {4\cC(\fp)^2} 
	\int_{\psi_{\rm min}(\fp)}^{\psi_{\rm max}(\fp)}\langle\hfp, \Re(\theta,\psi)\rangle_+ \Re(\theta,\psi) w_\fp(\psi)d\psi.
\end{equation*} 
By construction, the control vectors $\Re(\theta,\psi)$ defined in~\cref{eq_ControlVector} lie in the half plane $\bE_\theta$ defined by~\cref{eq_HalfPlane}. Therefore $\frac{\partial\cH}{\partial \hat \fp}(\fp,\hat \fp) \in \bE_\theta$, and thus $\cB(\fp) \subset \bE_\theta$ in view of \cref{eq_BEControlSet}, which fits with the discussion of the previous section and makes \cref{eq_BECS2D} well defined. This observation holds for arbitrary parameters $\cC(\fp)>0$, $\xi(\fp)>0$, $-\pi/2 < \psi_{\min}(\fp) < \psi_{\max}(\fp)<\pi/2$. 

Eventually, we  present the following result on the path curvature of any $T\cB$-admissible curve.
\begin{proposition}
\label{prop_BoundedCurvature}[Boundedness of Path Curvature]
Consider a $T \cB$-admissible curve $\wp=(\gamma,\eta)\in\Lip([0,1],\bM)$, w.r.t.\ the control sets $\cB$ defined in~\cref{eq_BEControlSet}.
Then the path curvature $\kappa$ of the physical projection curve $\gamma$ satisfies 
\begin{equation*}
\Im_{\rm min}(\wp(t))\leq\kappa(t)\leq \Im_{\rm max}(\wp(t)),\quad\forall t\in[0,1].
\end{equation*}
\end{proposition}
\begin{proof}
Define for each point $\fp = (\kp,\theta) \in \bM$ the following subset $\bE_\fp \subset \bE$ of the tangent space: 
\begin{equation*}
	\bE_\fp := \{ \lambda \Re(\theta,\psi)\mid \lambda \geq 0,\ \psi_{\min}(\fp) \leq \psi \leq \psi_{\max}(\fp)\}.
\end{equation*}
It is not hard to see that $\bE_\fp$ is a closed and convex two-dimensional cone, generated by two extremal vectors:
\begin{equation*}
\bE_\fp := \{ \alpha \Re(\theta,\psi_{\min}(\fp)) + \beta \Re(\theta,\psi_{\max}(\fp)) \mid \alpha,\beta \geq 0\}.
\end{equation*}
One has $\frac{\partial\cH}{\partial \hat \fp}(\fp,\hat \fp)\in \bE_\fp$, in view of the integral expression of this gradient and by convexity of $\bE_\fp$, and therefore $\cB(\fp) \subset \bE_\fp$ in view of \cref{eq_BEControlSet}.
As a result, for any $T\cB$-admissible curve $\wp=(\gamma,\eta) : [0,1] \to \bM$ and for a.e.\ $t \in [0,1]$ one has
\begin{equation*}	
\wp^\prime(t)\propto \Re\bigl(\eta(t),\varphi(t)\bigr),
\end{equation*}
for some angle $\varphi(t)$ such that $\psi_{\min}(\wp(t)) \leq \varphi(t) \leq \psi_{\max}(\wp(t))$. The curvature $\kappa$ of the path $\gamma$ is thus formulated as
\begin{equation*}
\kappa(t)=\frac{\eta^\prime(t)}{\|\gamma^\prime(t)\|}=\xi^{-1}\tan(\varphi(t))\in\bigl[\Im_{\rm min}(\wp(t)),\Im_{\rm max}(\wp(t))\bigr]
\end{equation*}
which concludes the proof.
\end{proof}

\begin{figure*}[!htbp]
\centering
\includegraphics[width=\linewidth]{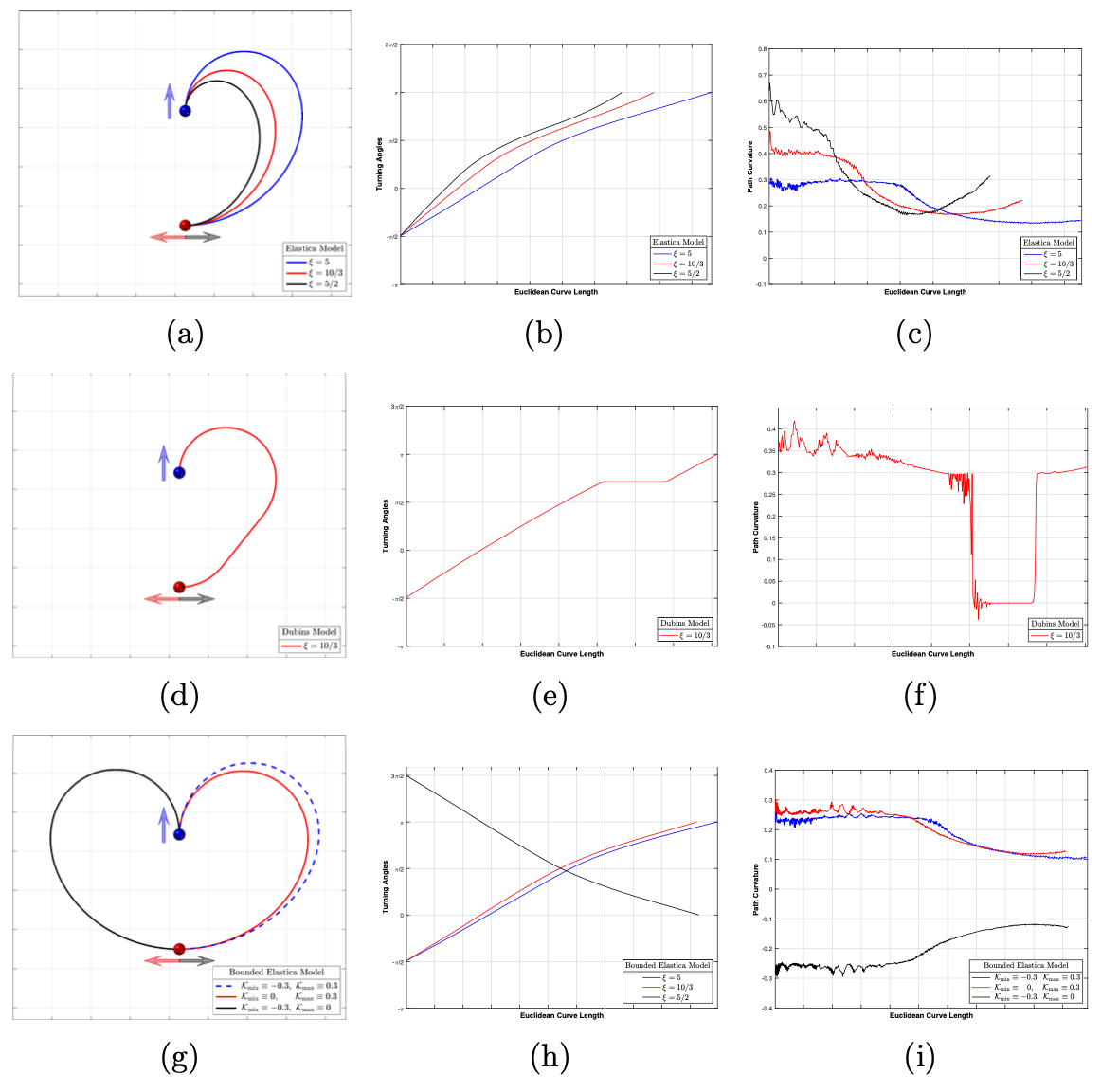}
\caption{Illustration of different curvature-penalized models on controlling the path curvature. \textbf{Column 1}:  Visualization for physical projection curves of the minimal paths associated to different models. The blue and red dots respectively indicate the source point and the target point, and the arrows indicate the path tangent directions assigned to these points.  \textbf{Columns 2-3}: The line plots of the turning angles and path curvature of the physical projection curves as functions of Euclidean curve length.}
\label{fig_AdvContrCurvature}
\end{figure*}

\subsection*{curvature-bounded Model is a Generalization of the Euler-Mumford Elastica Model} 
In this section, we reveal the tight relationship between the proposed curvature-bounded model and the Euler-Mumford Elastica model~\cite{chen2017global,mirebeau2018fast}, by comparing their respective Hamiltonians.
For that purpose, let us first reformulate the Hamiltonian $\cH$ of the curvature-bounded model, using the linear change of variables $\varphi = (\psi-\Psi_\fp)/\lambda_\fp$ in the defining integral of \cref{eq_BEHamiltonian}: 
\begin{equation}
\label{eq_EquivBEHamiltonian}
\cH(\fp,\hfp)=\frac{3}{8}\cC(\fp)^{-2}\int_{-\pi/2}^{\pi/2}\langle\hfp, \Re(\theta,\lambda_\fp \varphi+\Psi_\fp)\rangle_+^2\cos(\varphi)d\varphi
\end{equation}
for any point $\fp=(\kp,\theta)\in\bM$ and any co-vector $\hfp=(\hkp,\htheta)\in\bE^*$. Note that the weight $w_\fp$ in~\cref{eq_VarFejWeights} simplifies here to $\cos\varphi$. The control vectors in~\cref{eq_EquivBEHamiltonian} read
\begin{equation*}
\Re(\theta,\lambda_\fp \varphi+\Psi_\fp)=(\rn(\theta)\cos(\lambda_\fp \varphi+\Psi_\fp),\xi^{-1}\sin(\lambda_\fp\varphi+\Psi_\fp))
\end{equation*}
and therefore depend on the terms $\Psi_\fp$ and $\lambda_\fp$ related to the curvature bounds.
On the other hand, the Hamiltonian $\kH$ associated with the Euler-Mumford Elastica model can be formulated as follows~\cite{mirebeau2018fast} 
\begin{equation}
\label{eq_ElasticaHam}
\kH(\fp,\hfp)=\frac{3}{8}\cC(\fp)^{-2}\int_{-\pi/2}^{\pi/2}	\left\langle\hfp,\Re(\theta,\varphi) \right\rangle_+^2\cos\varphi\, d\varphi.
\end{equation}

The difference between the Hamiltonian $\cH$ of the proposed curvature-bounded model, and the original one $\kH$, thus solely lies in the translation and scaling $\lambda_\fp \varphi + \Psi_\fp$ of the angular parameter of the control vectors, introduced in this work to enforce curvature bounds.
In the special case where the angles $\psi_{\rm min},\,\psi_{\rm max}$ in the Hamiltonian $\cH$ are respectively set as $\psi_{\rm min}\equiv -\pi/2$ and $\psi_{\rm max}\equiv\pi/2$, the lower and upper curvature bounds become infinite $\Im_{\rm min}\equiv-\infty,\,\Im_{\rm max}\equiv\infty$, and the angular parameters of the control vectors remain untouched $\lambda_\fp=1,\,\Psi_\fp=0,\,\forall\fp\in\bM$. Under these circumstances, our model thus coincides with the standard Euler-Mumford Elastica model. 

Let  $\kB(\fp)$ be the control set of the Euler-Mumford Elastica model at a point $\fp = (\kp,\theta) \in \bM$. By a reasoning similar to the proposed curvature-bounded model, one has $\kB(\fp) \subset \bE_\theta$, and one may therefore define the sliced control sets
\begin{equation*}
\tilde\kB(\fp)=\left\{(\dot\nu,\dtheta)\in\bR^2~|~\dot\nu\geq0,(\dot\nu\rn(\theta),\dtheta)\in\kB(\fp)\right\}.
\end{equation*}
It is known~\cite{chen2017global,mirebeau2018fast} that $\tilde\kB(\fp)$ is an \emph{ellipse} (and more precisely a \emph{disk} in the special case where $\xi=1$), whose center lies on the $\dot \nu$ axis, and whose boundary is tangent to the origin.
\subsection*{On the parameter $\xi$ in the Curvature-bounded Model}
As defined in~\cref{eq_ControlVector} and~\cref{eq_BEHamiltonian}, the curvature-bounded model involves a parameter $\xi>0$ that modulates both the shape of the control sets and the relative importance of the curvature of the corresponding minimal paths. In Fig.~\ref{fig_ControlSets_Xi}, the boundaries of the sliced control sets $\tilde\cB(\fp)$ associated with the curvature-bounded model are illustrated, under varying values of parameter $\xi$ and of curvature bounds $\Im_{\rm min},\,\Im_{\rm max}$. Fig.~\ref{fig_ControlSets_Xi}d depicts a qualitative comparison of minimal paths computed using the curvature-bounded model. Both  minimal paths are generated using  the sliced control sets shown in Fig.~\ref{fig_ControlSets_Xi}a, corresponding to curvature bounds $\Im_{\rm min}\equiv-1$ and $\Im_{\rm max}\equiv1$. More precisely, the red line is computed using the curvature-bounded model with  $\xi=1$, and the black line with $\xi=2$. Similar to the its effect in the  Euler-Mumford Elastica model~\cite{chen2017global,mirebeau2018fast}, the parameter $\xi$ in the curvature-bounded model controls the smoothness of paths: high values of $\xi$ are expected to produce minimal paths whose physical projection curves have low maximum  curvature. Note that the increased curvature penalty associated with high values of $\xi$ is independent of curvature bounds. 
Similarly, in Fig.~\ref{fig_ControlSets_Xi}d, the red line exhibits a lower minimum turning radius (i.e. higher maximum path curvature) compared to the black line.

\subsection*{Significance in Controlling the Boundedness of the Path Curvature}
Given two endpoints as well as the angles that corresponds to the path tangents at these points, a major advantage of our curvature-bounded model,  over existing curvature-penalized models, lies at its strong ability in controlling the path curvature of the computed minimal paths, since it imposes a hard constraint on the curvature via arbitrary lower and upper bounds, as formulated in~\cref{eq_curvatureBounds}.  In contrast, the classical Euler-Mumford Elastica model~\cite{chen2017global,mirebeau2018fast} takes into account a squared curvature penalty as the model regularization, whereas the Dubins car model~\cite{mirebeau2018fast} is designed to search for minimal paths, where the \emph{absolute curvature} of the physical projection curves is bounded by a positively-defined scalar-valued function. This explicitly implies that neither the Euler-Mumford Elastica model nor the Dubins car model can limit the path curvature of the corresponding minimal paths to satisfy mandatory \emph{arbitrary bounds}. Moreover, the minimal paths of the Dubins model are non-smooth and frequently saturate the curvature bounds.

In Fig.~\ref{fig_AdvContrCurvature}, we conduct numerical experiments to exhibit the advantages of the curvature-bounded model in manipulating the curvature (of the physical projection curves) of the minimal paths, when comparing to the Euler-Mumford Elastica model and to the Dubins model. In this experiment, each test is set up with a single source point $\fs\in\bM$ and two target points $\fx_j$ for $j\in\{1,2\}$, which yields two minimal paths $\cG_{\fs,\fx_j}=(\tgamma_{\fs,\fx_j},\teta_{\fs,\fx_j})$ sharing the same source point $\fs$. We choose a point $\fx^*$ corresponding to lower minimum arrival time $\cT_{\cB}$, see~\cref{eq_MinimalTime}, i.e
\begin{equation}
\cT_{\cB}(\fs,\fx^*)=\min\{\cT_{\cB}(\fs,\fx_1),\cT_{\cB}(\fs,\fx_2)\}.
\end{equation}	
For the sake of simplicity, we denote by $\cG_*=(\tgamma_*,\teta_*)$ the chosen minimal path, where the physical projection curve, the turning angles and path curvature of the physical projection curve are respectively denoted by $\tgamma_*$, $\teta_*$ and $\tilde\kappa_*$. Furthermore, the cost $\cC$, as formulated in~\cref{eq_BEHamiltonian} and~\cref{eq_ElasticaHam}, for the curvature-penalized models considered are fixed as $\cC\equiv1$.

The left column of Fig.~\ref{fig_AdvContrCurvature} illustrates the physical projection curves $\tgamma_*$ derived from curvature-penalized models with respect to different parameters which affect the path curvature, the middle column draws the respective turning angles $\teta_*$ which is regarded as a function of the Euclidean curve length, and the right column illustrates the line plots of the path curvature of the physical projection $\tgamma_*$. More specifically, Fig.~\ref{fig_AdvContrCurvature}a illustrates the physical projection curves of the Euler-Mumford Elastica minimal paths, whose turning angles feature a slowly-varying property,  as demonstrated in Fig.~\ref{fig_AdvContrCurvature}b. It follows that the Euler-Mumford Elastica model encourages the minimal paths to be smooth. 
Analogous to the  Euler-Mumford Elastica model, the characteristics of the smoothness property can also be observed in Figs.~\ref{fig_AdvContrCurvature}g and~\ref{fig_AdvContrCurvature}h, where the physical projection curves and their turning angles are produced through the proposed curvature-bounded model.

\begin{figure*}[!htbp]
\centering
\includegraphics[width=0.99\linewidth]{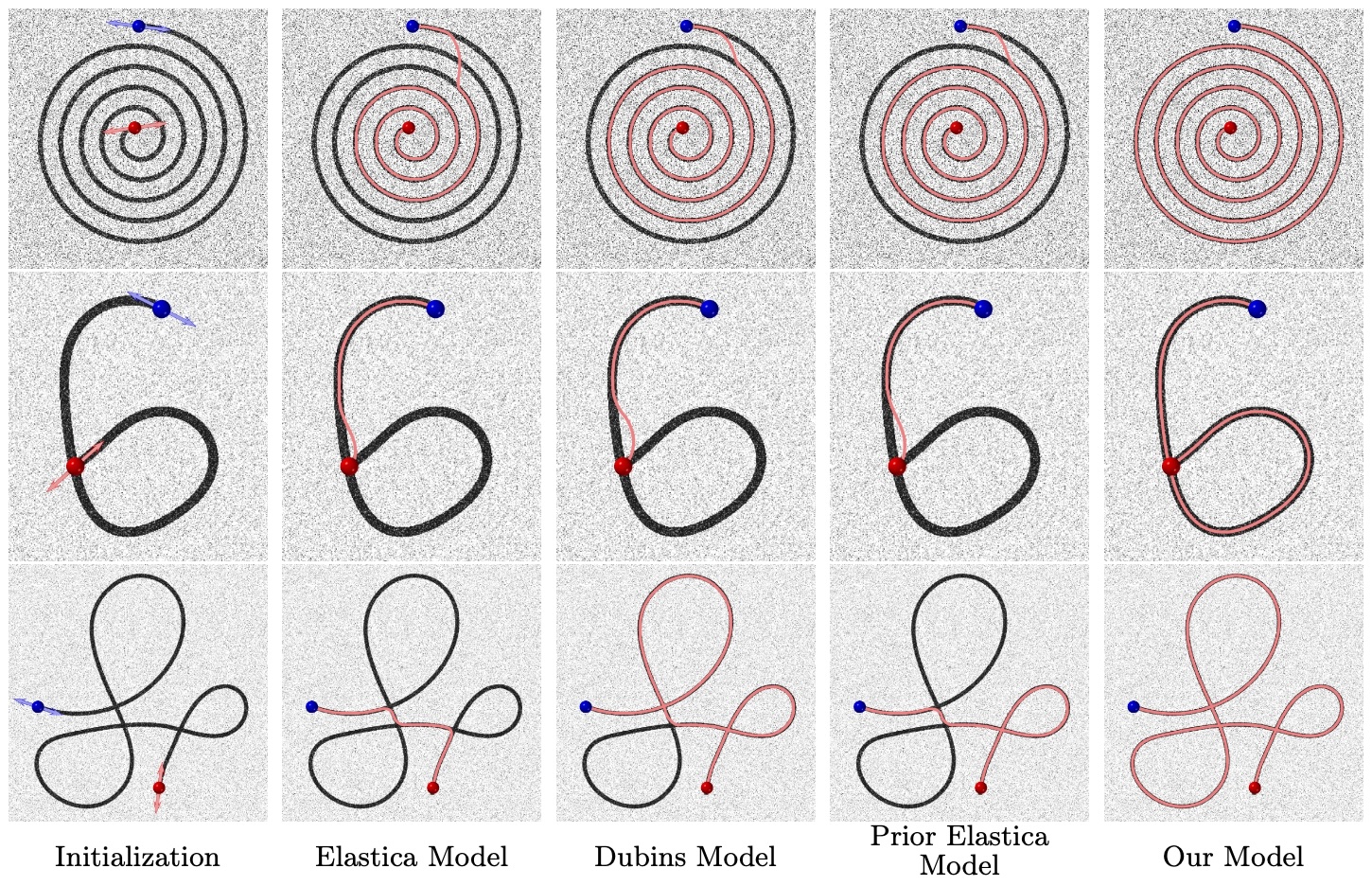}
\caption{Qualitative comparisons on synthetic images containing curvilinear structures of varying shapes. In column $1$, the red and blue dots respectively mark the physical positions of the source and target points for initializing the minimal path models, with  arrows indicating their assigned tangent vectors.  Columns $2$ to $5$ show the physical projection curves of minimal paths computed from the Euler-Mumford Elastica model, the Dubins model, the curvature prior Elastica model and the proposed curvature-bounded geodesic model, respectively .}
\label{fig_SynComp}
\end{figure*}
 
 
\subsection*{Curvilinear Structure Tracking}
\label{subsec_VesselTracking}
Finding curvilinear structure centerlines from images of various modalities is a challenging problem, due to the  presence of complex geometry appearance such as branched structures and rapidly-varying background content.  
Those centerlines  can be naturally modeled as  minimal paths which are solutions to a global optimization problem integrating data-driven cost with curve regularization~\cite{liao2022progressive,liao2023segmentation,chen2019minimal,pechaud2009extraction,li2007vessels}. In particular, the proposed curvature-bounded model takes both curvature penalization and curvature bounds as regularization, and is suitable for tracking curvilinear structure centerlines, since the curvature bounds can be efficiently estimated from the image data and can be naturally used as strong geometric priors to enhance the results. In particular, this  application works in conjunction with the estimation of curvature bounds from image data. 

In this article we focus on the extraction of curvilinear structure centerline,  providing that its two endpoints and  their respective tangent directions are given. The solution to such a task typically involves two major ingredients that should be estimated from the image data. Specifically, the first ingredient is the cost function $\cC$ as utilized in~\cref{eq_EquivBEHamiltonian}. It characterizes the appearance of the curvilinear structures and can be computed as a decreasing function of the orientation scores~\cite{chen2017global,chen2023computing}. In principle, the value of the cost $\cC(\fp)$ is supposed to be low at a given point $\fp=(\kp,\theta)\in\bM$,  provided that the physical position $\kp$ is inside a curvilinear structure and simultaneously the direction $(\cos\theta,\sin\theta)$ is nearly collinear to the centerline tangent at a position close to $\kp$.	
The second ingredient lies at the construction of the curvature bounds $\Im_{\rm min}$ and $\Im_{\rm max}$ using the image data. As introduced in~\cite{chen2023computing}, one can estimate the path curvature, which are referred to as curvature priors,  of computed curvilinear centerline segments as prescribed geometric features.
The curvature priors are encoded in a scalar-valued function $\varpi:\bM\to\bR$, by which  the curvature bounds $\Im_{\rm min}$ and $\Im_{\rm max}$ can be naturally constructed as follows:
\begin{equation}
\label{eq_CurvatureBounds}
\Im_{\rm max}(\fp):=\varpi(\fp)+\varsigma/2\quad\text{and}\quad\Im_{\rm min}(\fp):=\varpi(\fp)-\varsigma/2,
\end{equation}
where $\varsigma>0$ is a constant.

\begin{figure*}[!htbp]
\centering
\includegraphics[width=0.99\linewidth]{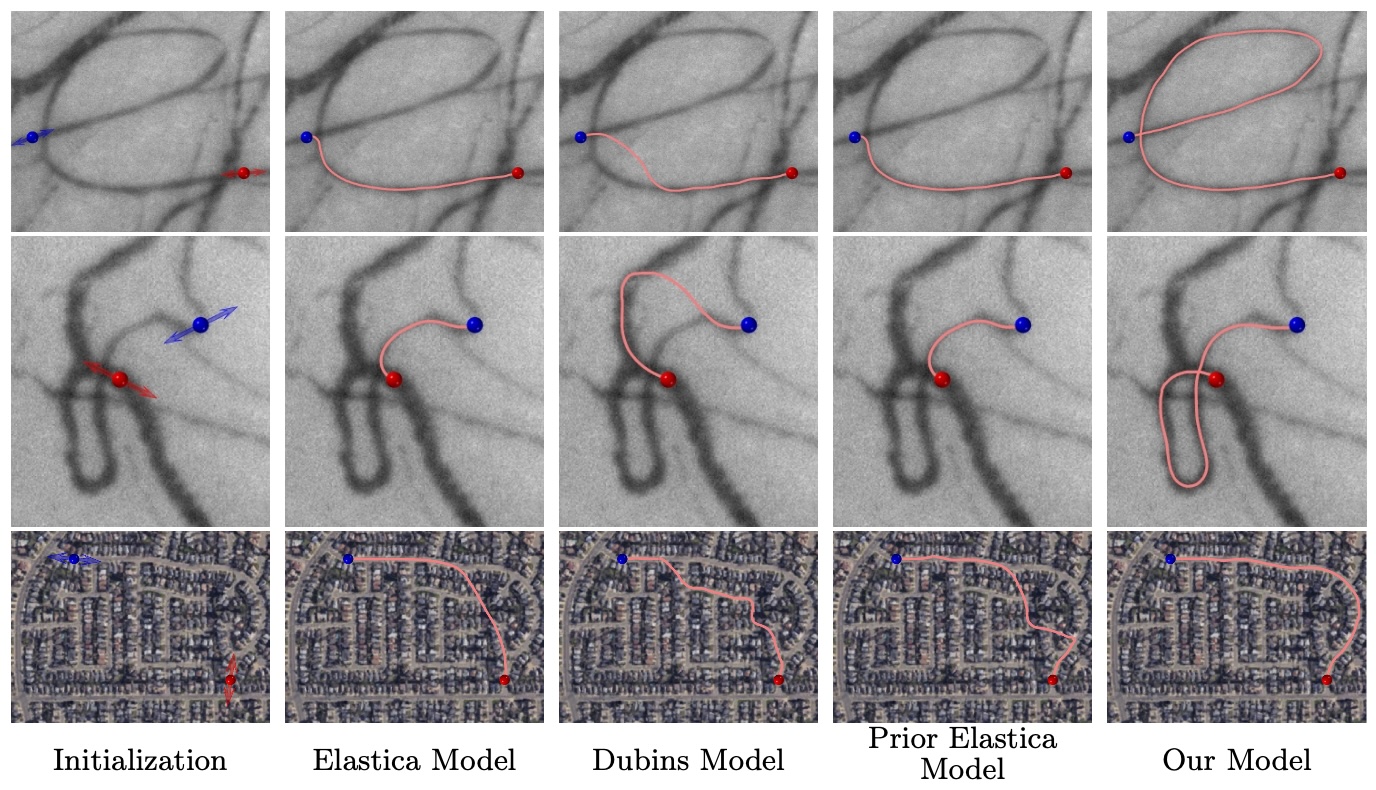}	
\caption{Qualitative comparison on tracking curvilinear structure centerlines from medical images. \textbf{Column 1} The dots and the arrows indicate the source and target points for initializing the minimal path models. \textbf{Columns 2-5} The physical projections of minimal paths derived from different curvature-penalized models are demonstrated.}
\label{fig_MedDataComp}
\end{figure*}


Note that the curvature priors $\varpi$ characterize the prescribed curvature values of curvilinear structure centerlines. We exploit the same procedure as the curvature prior Elastica model~\cite{chen2023computing} to compute $\varpi$ over the orientation-lifted space $\bM$, where a key step for this tracking method is to compute a set of disjoint piecewise smooth curves which fit to the discrete centerline segments. For this purpose, we develop an efficient segmentation-free algorithm for detecting those discrete centerline segments, allowing to fully benefit from the strongly oriented features of the elongated structures. The details can be seen in the Supplemental Information.

We demonstrate in Fig.~\ref{fig_SynComp} and Fig.~\ref{fig_MedDataComp} the qualitative comparison results between the proposed curvature-bounded model and state-of-the-art curvature-penalized models, involving the classical Euler-Mumford Elastica model~\cite{chen2017global,mirebeau2018fast}, the Dubins model~\cite{mirebeau2018fast} and the curvature prior Elastica model~\cite{chen2023computing}  in tracking curvilinear structure centerlines. In each test, the compared models are given two source points and two targets for initialization.  As in the literature~\cite{chen2017global,chen2023computing}, such an initialization manner  allows to generate four minimal paths, among which the one with minimum arrival time is chosen as the output of the model.

In Fig.~\ref{fig_SynComp}, the synthetic curvilinear structures are deliberately designed to facilitate the evaluation of the smoothness, rigidity and elasticity properties exhibited by different minimal path models. 
Columns $2$ to $5$ present the physical projection curves of minimal paths  computed using the Euler-Mumford Elastica model, the Dubins model and the  curvature-bounded model, respectively. The first row  of Fig.~\ref{fig_SynComp} shows a spiral-like curvilinear structure. As observed in columns in columns $2$ to $4$  of the first row,  the resulting minimal paths unfortunately suffer from shortcut artifacts,  failing to follow the desired structures.
In the second row, the image depicts an elongated shape of letter ``$6$", where the red dots indicate the physical endpoints of minimal paths placed around the junction pattern. The physical projection curves of the minimal paths are expected to traverse regions near the red dots in order to accurately delineate the entire structure. The optimal paths shown in columns $2$ and $3$ fail to accurately capture the circular structures. In the third row, the synthetic image presents a more complex curvilinear pattern composed of three self-intersecting loops.
In this test, the models shown in columns $2$ to $4$ fail to accurately extract the circular structures due to  insufficient constraint over the path curvature. In contrast, the curvature-bounded model as demonstrated in column $5$ of Fig.~\ref{fig_SynComp}, successfully produces minimal paths that capture these circular structures,  due to the  blended benefits from the curvature penalization and the mandatory  curvature bounds. 


The synthetic curvilinear patterns illustrated in Fig.~\ref{fig_SynComp} faithfully represent the characteristic geometries of blood vessels, nerve fibers, and roads in medical and aerial images,  as exemplified in Fig.~\ref{fig_MedDataComp}. Consistent with experiments conducted in Fig.~\ref{fig_SynComp}, compared curvature-penalized models exhibit obvious shortcut artifacts and fail to follow the desired curvilinear trajectories, as evident in columns $2$–$4$ of Fig.~\ref{fig_MedDataComp}. In contrast, the proposed curvature-bounded model robustly recovers the underlying structures, even in the presence of complex backgrounds or circular patterns, see column $5$.

\section{Methods}
We describe in this section the numerical method for computing minimal paths of the proposed curvature-bounded model. 
We introduce a static first-order HJB PDE associated with the optimal control problem of interest, a backtracking ordinary differential equation (ODE) for recovering optimal paths from its viscosity solution the minimal action map, and suitable finite difference discretizations of the PDE and ODE.

\subsection*{Upwind Discretization Scheme}
We apply the  scheme introduced in~\cite{mirebeau2018fast,mirebeau2019riemannian} for the discretization of the HJB PDE. Let $U$ be the numerical approximation, defined on the Cartesian grid $\bM_h=\bM\cap (h\bZ^2\times h\bZ/2\pi\bZ)$ of scale $h>0$, of the viscosity solution $\cU_\fs$ by solving in a single pass a well chosen wide stencil finite differences discretization of the HJB PDE \cref{eq_BEHJB}. In the numerical experiments, we set the grid scale to $h=2\pi/N_{\rm angles}$, where $N_{\rm angles}\in\mathbb{N}$ is the number of discrete angles along the third dimension of $\bM$.

Let us recall that the control vector in the proposed curvature-bounded model is formulated as $\Re(\theta,\psi)=(\cos(\psi)\rn(\theta),\xi^{-1}\sin(\psi))$.  
The integral defining the corresponding Hamiltonian $\cH$, see \cref{eq_EquivBEHamiltonian}, is first approximated using the Fejer quadrature rule: following \cite{mirebeau2018fast}
\begin{align}
\cC(\fp)^{2}\cH(\fp,\hfp)&=\frac{3}{8}\int_{-\pi/2}^{\pi/2}\langle\hfp, \Re(\theta,\lambda_\fp\varphi+\Psi_\fp)\rangle_+^2\,\cos\varphi\,d\varphi\nonumber\\
\label{eq_HamiltonianDecom}
&=\frac{3}{16}\sum_{1\leq\ell\leq L}f_\ell\langle\hfp,\Re(\theta,\tilde\varphi_\ell) \rangle_+^2+\cO(L^{-2})
\end{align}
for any point $\fp=(\kp,\theta)$. 
The number of $L$ nodes of the Fejer quadrature rule is fixed to $9$ in all the numerical experiments in this work. 
The Fejer weights $f_\ell \geq 0$, $1 \leq \ell \leq L$ are tabulated in the literature, and the angular nodes are defined as $\varphi_l := (2\ell- L - 1)\pi /(2L)$. For convenience we denoted $\tilde \varphi_\ell := \lambda_\fp\varphi_\ell+\Psi_\fp$.

In the spirit of the wroks~\cite{mirebeau2019riemannian,mirebeau2018fast}, each quadratic term $\langle\hfp,\Re(\theta,\tilde\varphi_\ell) \rangle_+^2$ in~\cref{eq_HamiltonianDecom} is approximated via a family of positive weights $\rho_{j\ell}^\theta:=\rho_j^\epsilon(\Re(\theta,\tilde\varphi_\ell))>0$ and offsets $\dfe^\theta_{j\ell}:=\dfe_j^\epsilon(\Re(\theta,\tilde\varphi_\ell))\in\bZ^3$ whose coordinates are integers, where $1\leq j \leq J:= 6$ are indices and where $\epsilon>0$ is a relaxation parameter, as follows: for any angular node $\tilde\varphi_\ell$ 
\begin{equation}
\label{eq_QuadraAppro}
\langle\hfp,\Re(\theta,\tilde\varphi_\ell)\rangle_+^2=\sum_{1\leq j \leq J}\rho_{j\ell}^\theta\langle\hfp,\dfe_{j\ell}^\theta \rangle _+^2+\|\hfp\|^2\cO(\epsilon^2).	
\end{equation}
The weights  $\rho_{j\ell}^\theta$ and offsets $\dfe^\theta_{j\ell}$ are computed relying on the theorem of Selling’s decomposition of positive quadratic forms~\cite{mirebeau2018fast}. 
Combining \cref{eq_HamiltonianDecom} and~\cref{eq_QuadraAppro}, we approximate the HJB PDE operator value $\cH(\fp,d\cU_\fs(\fp))$ as
\begin{align}
\label{eq_HamiltonQuadraDecom}
\cH(\fp,d\cU_\fs(\fp))\approx
\frac{3}{16}\cC(\fp)^{-2}\sum_{1\leq \ell \leq L}f_{\ell}\sum_{1\leq j \leq J}\rho^\theta_{j\ell}\langle d\cU_\fs(\fp),\dfe^\theta_{j\ell} \rangle_+^2.	
\end{align}
Finally, we approximate the directional gradient $\langle d\cU_\fs(\fp),\dfe^\theta_{j\ell}\rangle_+$ of the minimal action map with the upwind finite difference $h^{-1}(U(\fp)-U(\fp-h\dfe_{j\ell}^\theta))_+$ of the unknown $U:\bM_h\to[0,\infty[$ of our discretization. 
The HJB PDE is thus discretized by the system of equations  
\begin{equation*}
\frac{3}{8}\sum_{1\leq \ell\leq L}f_{\ell} \sum_{1\leq j\leq J}\rho_{j\ell}^\theta\left(\frac{U(\fp)-U(\fp-h\dfe_{j\ell}^\theta)}{h}\right)_+^2=\cC(\fp)^2,
\end{equation*}
for all $\fp \in \bM_h \setminus\{\fs\}$, together with the source constraint $U(\fs) = 0$ and outflow boundary conditions on $\partial \bM$. 

Using both of~\cref{eq_geodesicFlows} and~\cref{eq_HamiltonQuadraDecom},  the geodesic flows $\fV$ can be approximated through the positive weights and offsets of integer coordinates as follows
\begin{equation}
\label{eq_GVApprox}
\fV(\fp)\approx\frac{-3}{8\cC(\fp)^2}\sum_{1\leq \ell \leq L}f_\ell\sum_{1\leq j \leq J}\rho_{j\ell}^\theta \,\langle d\cU_\fs(\fp),\dfe_{j\ell}^\theta \rangle_+\,\dfe^\theta_{j\ell}.
\end{equation}
Then the  finite difference scheme for numerically approximating the geodesic flow  $\fV$ of the form formulated in~\cref{eq_GVApprox} is expressed as
\begin{equation*}
\fV(\fp)\approx\frac{-3}{8\cC(\fp)^2}\!\!	\sum_{1\leq \ell \leq L}f_\ell\sum_{1\leq j \leq J}\rho_{j\ell}^\theta\,\left(\frac{U(\fp)-U(\fp-h\dfe^\theta_{j\ell})} h\right)_+\dfe^\theta_{j\ell}.
\end{equation*}

%
%

\section{Conclusion}
\label{sec_conclusion}
In this paper, we propose a curvature-bounded geodesic model formulated within the HJB PDE framework that  computes globally optimal curves with second-order regularization under arbitrary curvature bounds constraint. The primary theoretical contribution is  the implicit embedding of  mandatory curvature range constraints into the Hamiltonian, which forms the foundation of the associated HJB PDE.  Furthermore, we develop an upwind discretization scheme for the Hamiltonian of the curvature-bounded model, yielding a discretized HJB system that seamlessly integrates with the state-of-the-art Hamiltonian Fast Marching method, enabling efficient approximation of minimal action maps.

The proposed curvature-bounded model effectively  tracks curvilinear structures in complex scenarios, where the construction of the curvature bounds, defining the admissible curvature varying ranges,  is a critical step for practical applications.In our formulation, the curvature bound at each point is defined by  a prior term combined with a scalar  offset term that adapts to local  features of elongated shapes. These curvature priors are obtained by fitting smooth curves to disjoint segments of predefined curvilinear centerlines. However, their accuracy  may be degraded by fragmentation in the precomputed curvilinear centerlines. An alternative and promising direction is to design a path voting scheme, in which curvature priors and offset terms are efficiently estimated from a dense set of minimal paths distributed across all plausible curvilinear patterns.

Moreover, the  proposed curvature-bounded geodesic model is a powerful tool for extracting curvilinear structures, providing that their endpoints and the tangent directions at those points are given. Future work will focus on integrating this model with a path classifier~\cite{gupta2023topology,turetken2016reconstructing} to enable the automatic reconstruction and tracking of entire curvilinear networks.


\section*{Appendices}
\appendix
\section{Computing the Minimal Action Map Through the Static First-order HJB PDE Framework}
For clarity, we start with a reformulation of the optimal control problem of~\cref{eq_MinimalTime}, involving the control sets $\cB$, as the computation of a path length distance \cite{mirebeau2019riemannian,bardi1997Optimal}, from a source point $\fs\in\bM$ to a target point $\fp\in\bM$.
Indeed, the data of the control sets $\cB$ is equivalent to the data of a (sub-Finslerian) geodesic metric $\cF:\bM\times\bE\to[0,\infty]$ defined as 
\begin{equation}
\cF(\fp,\dfp):=\inf\,\bigl\{\lambda>0~|~(\dfp/\lambda)\in\cB(\fp)\bigr\}	.
\end{equation}
The corresponding path length reads
\begin{equation}
\label{eq_Length}
\Length_\cF(\wp):=\int_0^1\cF(\wp(t),\wp^\prime(t))dt,
\end{equation}
and its minimization leads to an equivalent expression of the minimal traveltime from $\fs$ to $\fp$ w.r.t.\ $\cB$, see \cref{eq_MinimalTime} 
\begin{align}
\label{eq_MinimalLength}
\cT_{\cB}(\fs,\fp):=\inf\bigl\{\Length_\cF(\wp)~|~&\wp\in\Lip([0,1],\bM),\nonumber\\
&\wp(0)=\fs,\wp(1)=\fp\bigr\}.	
\end{align}

Fixing the source point $\fs$, the minimal action map $\cU_\fs:\bM\to [0,\infty]$ assigns to each possible endpoint $\fp$ the corresponding minimum arrival time:
\begin{equation}
\label{eq_MAP}
\cU_\fs(\fp)=\cT_\cB(\fs,\fp).
\end{equation}
The static first-order HJB PDE framework~\cite{cohen1997global,bardi1997Optimal} characterizes $\cU_\fs$ as the unique viscosity solution to a boundary value problem:
\begin{equation}
\label{eq_BEHJB}
\begin{cases}
\cH(\fp,d\cU_\fs(\fp))=\frac{1}{2},&\forall \fx\in\bM\backslash\{\fs\}\\
\cU_\fs(\fs)=0,&\text{(boundary condition)},
\end{cases}
\end{equation}
where $d\cU_\fs$ denotes the first-order differential of the minimal action map $\cU_\fs$. 

The geodesic flow towards the seed $\fs$
is represented by a vector field $\fV$ defined in terms of the minimal action map $\cU_\fs$:
\begin{equation}
\label{eq_geodesicFlows}
\fV(\fp)=-\partial_2 \cH(\fp,d\cU_\fs(\fp))
\end{equation}
where $\partial_2\cH(\fp,\hfp):=\partial\cH(\fp,\hfp)/\partial\hfp$. The vector $\fV(\fp)$ is the negative gradient of the minimal action map $\cU_\fs$ at $\fp$ and w.r.t.\ the local metric $\cF(\fp,\cdot)$, and its characterizes the tangent direction of the shortest path towards $\fs$. 
More precisely, consider an arbitrary point $\fp \in \bM$, and the path $\cG : [0,T] \to \bM$ obeying the ordinary differential equation
\begin{equation*}
\label{eq_Backtracking}
\cG^\prime(t)=\fV(\cG(t)),~\forall t\in]0,T[.
\end{equation*}
with initial value $\cG(0)=\fp$ and total time $T = \cU_\fs(\fp)$. 
This \emph{backtracked} path $\cG$ travels by construction from the target point $\fp$ to the source point $\fs$. By re-parametrization one obtains a minimal path $\cG_{\fs,\fp}(t):=\cG(T(1-t))$ from the source $\cG_{\fs,\fp}(0)=\fs$ to the target $\cG_{\fs,\fp}(1)=\fp$. 

\providecommand{\bI}{\mathbb I}
\providecommand{\fone}{\mathbf{1}}
\providecommand{\fM}{\mathbf M}
\providecommand{\fr}{\mathbf r}
\providecommand{\fb}{\mathbf b}
\providecommand{\fA}{\mathbf A}
\providecommand{\fy}{\mathbf y}
\providecommand{\fn}{\mathrm{n}}
\providecommand{\tw}{\tilde{w}}
\providecommand{\tu}{\tilde{u}}
\providecommand{\tI}{\tilde{I}}
\providecommand{\tCU}{\tilde{\mathcal{U}}}
\providecommand{\OLP}{\bm{\rho}}
\providecommand{\OLT}{\bm{\eta}}
\providecommand{\ky}{\mathfrak y}
\providecommand{\kL}{\mathfrak L}
\providecommand{\kT}{\mathfrak T}
\providecommand{\cK}{\mathcal K}
\providecommand{\cE}{\mathcal E}
\providecommand{\cP}{\mathcal P}
\providecommand{\cS}{\mathcal S}
\providecommand{\cD}{\mathcal D}
\providecommand{\cM}{\mathcal M}
\providecommand{\cN}{\mathcal N}
\providecommand{\cV}{\mathcal V}
\providecommand{\cX}{\mathcal X}
\providecommand{\tka}{\tilde\kappa}
\providecommand{\cW}{\mathcal W}
\providecommand{\hkx}{\hat{\mathfrak{x}}}
\providecommand{\hks}{\hat{\mathfrak{s}}}
\providecommand{\rTC}{{\rm TC}}
\providecommand{\rEM}{{\rm EM}}
\providecommand{\rmD}{{\rm D}}
\providecommand{\rRS}{{\rm RS}}
\providecommand{\ts}{\textstyle}
\providecommand{\vp}{\varphi}
\providecommand{\ve}{\varepsilon}
\providecommand{\ray}{\Re_z(p)}
\providecommand{\trans}{\top}
\providecommand{\dx}{\dot x}
\providecommand{\dy}{\dot y}
\providecommand{\ds}{\dot s}
\providecommand{\dfr}{\dot{\mathbf r}}
\providecommand{\dfx}{\dot{\mathbf x}}
\providecommand{\dfnt}{{\dot{\mathbf n}}_\theta}
\providecommand{\hx}{\hat{x}}
\section{Computing Geometric Features from Curvilinear Structures}
\subsection*{Orientation Scores}
In practice, the tool of the orientation scores is designed to characterize the appearance of curvilinear structures in an anisotropy manner, so as to alleviate the negative influence from the mixed complicated data distribution and  structures of crossing morphology. Recall that $\bM=\Omega\times\bS^1$ represents the orientation-lifted space, where $\Omega\subset\bR^2$ is the image domain and $\bS^1=[0,2\pi[$  is an interval of periodic boundary condition. 
The orientation score map,  denoted by a scalar-valued map $\tilde\alpha:\bM\to[0,\infty[$, can be extracted from a gray level image $f:\Omega\to\bR$ via an orientation-dependent filter bank $\hbar$:
\begin{equation*}
\tilde\alpha(\kp,\theta)=(\hbar_\theta\ast f)(\kp)
\end{equation*}
for any physical position $\kp\in\Omega$ and for any an angular coordinate $\theta\in\bS^1$, where $\ast$ is a convolution operator. In practice, the orientation score map $\tilde\alpha$ is usually normalized to the range $[0,1]$:
\begin{equation}
\label{eq_NLOS}
\alpha(\kp,\theta)=\tilde\alpha(\kp,\theta)/\|\tilde\alpha\|_\infty.
\end{equation}
Many curvilinear structure filters can be applied to compute the orientation score map $\alpha$, where typical examples may involve the steerable filters~\cite{law2008three,jacob2004design}, the filter depending on the cake wavelets~\cite{franken2009crossing} and filter consisting of multiple Gaussian kernels~\cite{moriconi2018inference}. The value of the orientation score  $\alpha(\kp,\theta)$ characterizes the probability  that the physical position $\kp$ is inside the curvilinear structures and simultaneously the direction $\fn(\theta)=(\cos\theta,\sin\theta)$ is collinear to the orientation that a curvilinear structure should have at the physical position $\kp$. In Fig.~\ref{fig_OS}, we illustrate an example for the orientation scores computed from an image involving two curvilinear structures which cross one another.

\subsection*{Detection of the  Centerlines of Curvilinear Structures}
We introduce a method for extracting the centerlines of curvilinear structures using the geometric descriptor of orientation scores. The basic idea is to regard the centerlines of curvilinear structures as a collection of optimal points of orientation score map $\alpha$. In addition, these optimal points should also pass several tests, which are detailed in the following paragraphs. 

Firstly, we establish a binary-valued function in the orientation-lifted space, denoted by $\Phi:\bM\to\{0,1\}$, such that fixing a position $\kp$ the value $\Phi(\kp,\theta)=1$ implies that the orientation score $\alpha(\kp,\theta)$ is a local maximum of $\alpha(\kp,\cdot)$ along the angular dimension, i.e. given a proper value $\epsilon>0$, for any angle that $\phi\in[\theta-\epsilon,\theta+\epsilon]\text{~and~}\phi\neq\theta$, one has $\alpha(\kp,\theta)>\alpha(\kp,\phi)$. In other words, $\Phi(\kp,\cdot)$ characterizes the optimal angles at the physical position $\kp$ in the sense of the orientation scores. In case the position $\kp$ is located at a curvilinear structure centerline and $\Phi(\kp,\theta_*)=1$ at some angle $\theta_*$, then the direction $\fn(\theta_*)=(\cos\theta_*,\sin\theta_*)$ should be collinear to the tangent of the centerline at $\kp$.

Secondly, for each orientation-lifted point $\fp=(\kp,\theta)\in\bM$ such that $\Phi(\kp,\theta)=1$, we detect two points $(\kp_1,\theta_1)$ and $(\kp_2,\theta_2)$ which are  close to $\fp$. In our work this is  implemented in a two-step procedure. Specifically, the first step is to produce the physical positions $\kp_1$ and $\kp_2$, respectively by moving the physical position $\kp$ along the directions $\vp(\theta)=(-\sin\theta,\cos\theta)$ and $-\vp(\theta)=(\sin\theta,-\cos\theta)$, reading as
\begin{align*}
&\kp_1=\kp+\iota\vp(\theta)\\
&\kp_2=\kp-\iota\vp(\theta),
\end{align*}
where $\iota>0$ is an offset parameter. 
The angles $\theta_1$ (resp. $\theta_2$) corresponds to the local maximum of the orientation scores $\alpha(\kp_1,\cdot)$ (resp. $\alpha(\kp_2,\cdot)$) in  the angular dimension within the range $[\theta-\zeta,\theta+\zeta]$ of a periodic boundary condition, where $\zeta>0$ is a positive constant. In other words, the objective angles $\theta_1$ and $\theta_2$ are detected as
\begin{align*}
&\theta_1=\underset{\phi\in[\theta-\zeta,\theta+\zeta]}{\arg\max}~\alpha(\kp_1,\phi)\\
&\theta_2=\underset{\phi\in[\theta-\zeta,\theta+\zeta]}{\arg\max}~\alpha(\kp_2,\phi).
\end{align*}
Then we obtain a new function $\Phi_1:\bM\to\{0,1\}$ which is defined as
\begin{equation*}
\Phi_1(\kp,\theta)=
\begin{cases}
1,&\text{if~}\alpha(\kp,\theta)>\max\{\alpha(\kp_1,\theta_1),\alpha(\kp_2,\theta_2)\}~\text{and}~\Phi(\kp,\theta)=1 \\
0,&\text{otherwise}.
\end{cases}	
\end{equation*}
Actually, the function $\Phi_1(\kp,\theta)=1$ means that the physical position $\kp$ is a candidate point of the centerline of a curvilinear structure, since it is a locally optimal point, in terms of the orientation score map $\alpha$, along the normal direction of the centerline at $\kp$. In order to reduce the negative influence from image noise and also to further refine the detected locally optimal points, we also invoke the test considered in the literature~\cite{wang2013interactive} related to the image gradients $\vartheta:\Omega\to\bR^2$ of a gray level image $f$, i.e.
\begin{equation*}
\vartheta(\kp)=(\nabla G_\sigma\ast f)(\kp)
\end{equation*}
where $G_\sigma$ is a Gaussian kernel whose standard deviation is $\sigma$ and where $\nabla G_\sigma$ denotes its standard Euclidean gradient in the space $\bR^2$. As implemented in~the literature~\cite{wang2013interactive}, each physical position $\kp$ is assigned to an optimal angle $\theta^*$
\begin{equation*}
\theta^*:=\underset{\phi\in[0,2\pi[}{\arg\max}~\alpha(\kp,\phi),	
\end{equation*}
which characterizes the  orientation of the curvilinear structure at the physical position $\kp$. The original test presented in~\cite{wang2013interactive} says that a point $\kp$ is a centerline point if there exists a radius $r$ such that the pair $(\kp,\theta^*)$ passes the test 
\begin{equation}
\label{eq_NotUsedTest}
\vartheta\bigl(\kp+r\vp(\theta^*)\bigr)=-\vartheta\bigl(\kp-r\vp(\theta^*)\bigr).	
\end{equation}
In this work,  we consider a new test, slightly different to the one in~\eqref{eq_NotUsedTest},  which utilizes all the orientation-lifted points $(\kp,\theta)$ such that $\Phi_1(\kp,\theta)=1$. In particular, we evaluate the following formulation
\begin{equation}
\label{eq_GradTest}
\sign\left(\langle\vp(\theta),\vartheta(\kp+r\vp(\theta))\rangle\right)=-\sign\left(\langle\vp(\theta),\vartheta(\kp-r\vp(\theta))\rangle\right).
\end{equation}
We define the target indicator $\Phi_2:\bM\to\{0,1\}$ of the centerlines of curvilinear structures, where  $\Phi_2(\kp,\theta)=1$ if $\Phi_1(\kp,\theta)=1$ and if the point $(\kp,\theta)$ passes the test in~\eqref{eq_GradTest}. Finally, we obtain a map $\zeta:\Omega\to\{0,1\}$ that involves all the admissible centerline points 
\begin{equation}
\label{eq_SkeleMap}
\zeta(\fp)=
\begin{cases}
1,&\text{if~}\int_0^{2\pi}\Phi_2(\kp,\theta)d\theta>0\\
0,&\text{otherwise}.
\end{cases}	
\end{equation}
The binary-valued map $\zeta$ characterizes the curvilinear structure centerlines involved in the image data. It will be further processed to generate a family of disjoint centerline segments, as presented in next section.

\section{Computing the Curvature Prior Map \texorpdfstring{$\varpi$}{varpi} from Curvilinear Structure Centerlines}
In the application of tracking curvilinear structure centerlines using the introduced bounded Elastica model,  a crucial ingredient is the construction of the curvature prior map $\varpi$ using the path curvature which is the intrinsic geometric properties of the curvilinear structure centerline candidates, so as to compute the curvature bounds $\Im_{\rm min}$ and $\Im_{\rm max}$. 
In our work,  we follow the efficient method proposed in the literature~\cite{chen2023computing} to compute the curvature prior map $\varpi$, where the first step is to fit piecewise smooth curvature-penalized optimal paths to the disjoint discrete centerline segments using the map $\zeta$, as introduced in the following section. 

\subsection*{Fitting Smooth Curvature-penalized Minimal Paths to Disjoint Centerline Segments of Curvilinear Structures}
Let  $\Omega_h:=\Omega\cap \bZ^2$ be a  Cartesian grid, where $h$ is the grid scale. Without  loss of generality, we set the scale $h=1$.  In this discrete setting, a discrete centerline segment is defined as a family of ordered grid points of eight-connection and a junction point is regarded as a particular grid point which has more than two eight-connected neighbouring points. 

The centerline indicator map $\zeta$ formulated in~\eqref{eq_SkeleMap} involves the information of the  potential discrete centerline segments. In order to guarantee that the width of each individual centerline segment equals exactly  one grid point, we apply the morphological filter to the set $\{\kp\in\Omega_h~|~\zeta(\kp)>0\}$ and then remove all the junction points to generate $N$ disjoint centerline segments $\Gamma_j$ indexed by $1 \leq j \leq N$. Fig.~\ref{fig_CurvaturePrior}a illustrates those discrete centerline segments using  different colors. From them one can also construct $N$ disjoint tubular neighbourhood regions $\cT_j\subset\Omega$ for $1\leq j \leq N$ in terms of Euclidean distances. This is to say
\begin{equation}
\cT_j=\{\kp\in \Omega~|~d(\kp,\Gamma_j)<d(\kp,\Gamma_i),\,\forall i\neq j\},	
\end{equation}
where $d(\kp,\Gamma_j)$ denotes the Euclidean distance between a physical position $\kp$ and the discrete centerline segment $\Gamma_j$. In this way, one can point out that for any $i \neq j$, the neighbourhoods $\cT_i$ and $\cT_j$ are disjoint, i.e.,  $\cT_i\cap  \cT_j=\emptyset$. In Fig.~\ref{fig_CurvaturePrior}b, we illustrate an example of these disjoint neighbourhood regions.

In the open bounded and connected domain $\cT_j$, each individual discrete centerline segment $\Gamma_j$  can provide two endpoints $\kp_j,\,\kq_j\in\Omega$, i.e. an endpoint only has a single neighbouring point. With these definitions in hands, we attempt to minimize the following path energy in order to track the path $\cG_j$
\begin{equation}
\label{eq_CurvatureEnergy}
\int_0^1\cC(\gamma(t),\eta(t))\kC(\xi\kappa(t))dt,\quad\text{subject to}~
\begin{cases}
\gamma(0)=\kp_j,&\\
\gamma(1)=\kq_j,&\\
\gamma(t)\in \cT_j,&~\forall t\in[0,1].	
\end{cases}
\end{equation}
where $\cC$ is the image data-driven cost function, $\kC$ is a cost of the path curvature $\kappa$ and $\xi>0$ is a weighting parameter on the curvature. More specifically, we define the data-driven cost function $\cC$ as follows:
\begin{equation*}
\cC(\kp,\theta)=\exp\left(-\beta\alpha(\kp,\theta)\right),
\end{equation*}
for any physical position $\kp\in\Omega$ and any angular coordinate $\theta\in\bS^1$. One can point out that the cost $\cC$ is a decreasing function of the orientation scores $\alpha$. It is also used as the data-drive cost function of the proposed bounded Elastica model. 
Moreover, the curvature cost function $\kC$ is dependent to the curvature-penalized minimal path model considered. For instances, the cost $\kC(a)=1+a^2$ for any $a\in\bR$ corresponds to the Euler-Mumford  elastica model~\cite{chen2017global,mirebeau2018fast} and  $\kC(a)=\sqrt{1+a^2}$ corresponds to the Reeds-Shepp optimal curve model~\cite{duits2018optimal}.

The minimization of the weighted curve length defined in~\eqref{eq_CurvatureEnergy} can be implemented via the HJB PDE framework. Specifically,one can estimate the minimal action map $\cU_j$ by addressing the  HJB PDE by the Hamiltonian Fast-Marching method~\cite{mirebeau2018fast} or by the GPU-implemented way~\cite{mirebeau2023massively}. Then a gradient descent procedure is performed on the minimal action map $\cU_j$ to track the minimal path $\cG_j=(\gamma_j,\eta_j)$ lying inside the corresponding tubular neighbourhood $\cT_j$. Note that during the estimation of the minimal action map $\cU_j$, the set $\partial\cT_j\times\bS^1$ of the tubular neighbourhood  is used as a wall to limit the distance propagation within the domain $\cT_j \times \bS^1$.
In Fig.~\ref{fig_CurvaturePrior}, we illustrate this procedure using the image data shown in Fig.~\ref{fig_OS}a.

\subsection*{The Computation of the Curvature Prior Map $\varpi$}
Once all the minimal paths $\cG_j=(\gamma_j,\eta_j)$ are generated, we can estimate their curvature $\kappa_j$ of each physical projection curve $\gamma_j$ as follows:
\begin{equation*}
\kappa_j(u)=\frac{\eta_j^\prime(u)}{\|\gamma_j^\prime(u)\|},
\end{equation*}
for any parameter $u\in[0,1]$. 

Eventually, we follow the method introduced in the literature~\cite{chen2023computing} to compute  the curvature prior map $\varpi:\bM\to\bR$ using the computed physical projection curves $\{\gamma_j\}_j$ and their curvature $\{\kappa_j\}_j$, as shown in Fig.~\ref{fig_CurvaturePrior}d. In this figure, the values of $\kappa_j$ are visualized by different colors, where the red arrows indicate the parameterization of the respective planar components $\gamma_j$ of the minimal paths $\cG_j$. Following~\cite{chen2023computing}, we define by $0\leq\Lambda\leq \min_j\{\|\kappa_j\|^{-1}_\infty\}$ a bounding parameter and let $\cN_j$ be the unit normal to the curve $\gamma_j$.  Then one can construct the curvature prior map $\varpi$ by   
\begin{equation*}
\varpi(\kp,\theta)=
\begin{cases}
\kappa(u)\sign\left(\langle\fn(\theta),\fn(\eta_j(u)) \rangle\right),&\text{if~}\kp\in\cT_j~\text{and}~\kp=\gamma_j(u)+\lambda\cN_j(u),\,\forall\lambda\in[-\Lambda,\Lambda]\\
0,&\text{otherwise}.
\end{cases}
\end{equation*}
or by 
\begin{equation*}
\varpi(\kp,\theta)=
\begin{cases}
\kappa(u)\langle\fn(\theta),\fn(\eta_j(u)) \rangle,&\text{if~}\kp\in\cT_j~\text{and}~\kp=\gamma_j(u)+\lambda\cN_j(u),\,\forall\lambda\in[-\Lambda,\Lambda]\\
0,&\text{otherwise}.
\end{cases}
\end{equation*}
Note that $\gamma_j+\lambda\cN_j$ represent an offset curve of $\gamma_j$. We refer to the literature~\cite{chen2023computing} for more detail on the computation of the curvature prior map $\varpi$.

\section{The Dubins Model with Extended Curvature Bounds}
In contrast to the introduced bounded Elastica model which takes into account arbitrary curvature bounds as a geometric prior, the Dubins model~\cite{mirebeau2018fast} alternatively invokes a boundedness limitation to the absolute path curvature.  As introduced in the work~\cite{mirebeau2018fast}, the Hamiltonian $\kH^{\rm D}$ of the original Dubins model is formulated as 
\begin{align}
\label{eq_DubinsHamiltonian}	
\kH^{\rm D}(\fp,\hfp)&=\frac{1}{2}\max\left\{0,\,\left\langle\hfp,\left(\fn(\theta),\xi^{-1}\right) \right\rangle,\left\langle\hfp,\left(\fn(\theta),-\xi^{-1}\right)\right\rangle\right\}^2\nonumber\\
&=\frac{1}{2}\max\left\{\left\langle\hfp,\left(\fn(\theta),\xi^{-1}\right) \right\rangle_+,\left\langle\hfp,\left(\fn(\theta),-\xi^{-1}\right)\right\rangle_+\right\}^2
\end{align}
for any point $\fp=(\kp,\theta)\in\bM$ and for any co-vector $\hfp=(\hkp,\htheta)\in\bR^2\times\bR$. 

The control set $\kB^{\rm D}$ of the Dubins model can be formulated as follows~\cite{mirebeau2018fast}
\begin{align*}
\kB^{\rm D}(\fp)=&\left\{(\dot{\kp},\dtheta);\xi|\dtheta|\leq\|\dot\kp\| \leq 1,\dot\kp=\|\dot\kp\|\fn(\theta)\right\}\\
=&\left\{a\fn(\theta),b\xi^{-1});0\leq |b|\leq a \leq 1\right\}.
\end{align*}
As discussed in~\cite{mirebeau2018fast}, the control set $\kB^{\rm D}(\fp)$ is a triangle whose vertices are $\mathbf{0}$, $(\fn(\theta),1/\xi)$ and $(\fn(\theta),-1/\xi)$. For any minimal path $\cG=(\tilde\gamma,\tilde\eta):[0,1]\to\bM$, the curvature $\tilde\kappa=\tilde\eta^\prime/\|\tilde\gamma^\prime\|$ satisfies
\begin{equation*}
\tilde\kappa(u)\in\left[-\xi^{-1},\xi^{-1}\right],\quad \forall u\in[0,1]	
\end{equation*}
Furthermore, for the purpose of visualization, we consider the following sliced control sets 
\begin{equation*}
\tilde{\kB}^{\rm D}(\fp)=\{(\dot\nu,\dtheta)\in\bR^2;\dot\nu>0,(\dot\nu\fn(\theta),\dtheta)\in\kB^{\rm D}(\fp)\},	
\end{equation*}
where we recall that $\fn(\theta)$ is defined as $\fn(\theta)=(\cos\theta,\sin\theta)$. In Fig.~\ref{fig_DubinsControlSet}, we illustrate the sliced control set $\kB^{\rm D}(\fp)$ at a point $\fp$, constructed with parameters $\xi=1$, $\xi=2$ and $\xi=3$.

In this original Dubins model, the constraint imposed to the absolute path curvature is implemented via the constant parameter $\xi$. However, such a setting may lose the information from the image data. As pointed out in the literature~\cite{mirebeau2019hamiltonian}, the constant parameter $\xi$ can be extended as a pointwise function  regarded as the absolute curvature bounds of the Dubins car model. For this purpose, we consider  a scalar-valued positively-defined function $\cK:\bM\to\bR^+$, defined by
\begin{equation}
\label{eq_DubinsBounds}
\cK(\fp)=\max\bigl\{|\Im_{\rm min}(\fp)|,\,|\Im_{\rm max}(\fp)|\bigr\},	
\end{equation}
where $\Im_{\rm min},\Im_{\rm max}:\bM\to\bR$ are the curvature bounds used in the introduced bounded Elastica model.

\begin{figure}[htbp]
\centering
\includegraphics[width=0.9\linewidth]{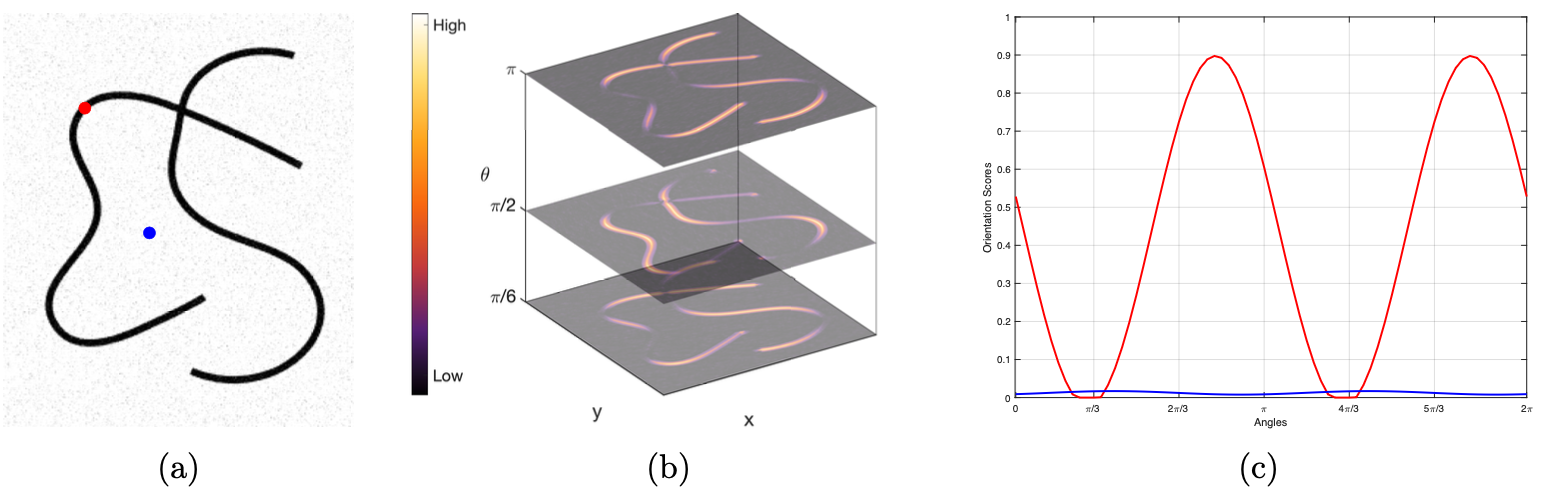}
\caption{Visualization for the orientation scores of curvilinear structures. (\textbf{a}) The raw image for estimating the orientations score map by the optimally oriented flux filter as introduced in the literature~\cite{law2008three}. The red and blue dots represent two sampled points, which are respectively located at the curvilinear structure and at the background.  (\textbf{b}) Visualization of orientation scores $\alpha(x,\theta)$ at three slices with respect to three sampled angles $\theta=\pi/6,\,\pi/2$ and $\pi$. (\textbf{c}) Plots of the normalized orientation score values $\alpha$ at two sampled points, where the red and blue lines  represent the values of $\alpha$ at the red and blue dots in figure (a), respectively.}
\label{fig_OS}
\end{figure}

\begin{figure}[htbp]
\centering
\includegraphics[width=0.95\linewidth]{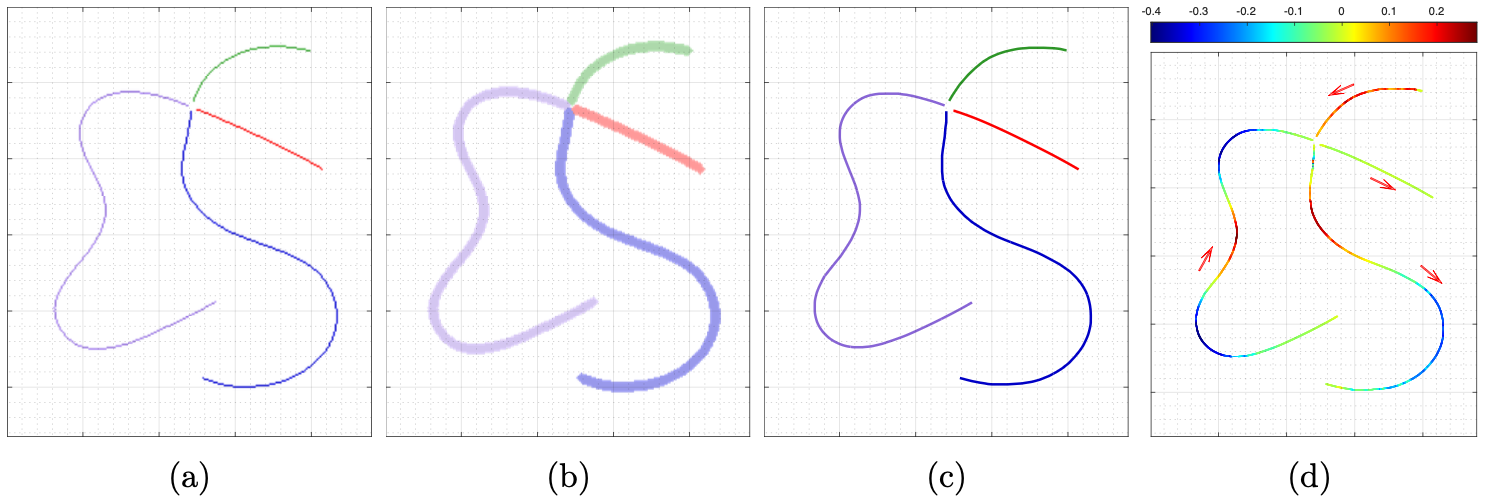}
\caption{Illustration for the computation of the curvature prior map $\varpi$. (\textbf{a}) Visualization of the piecewise disjoint discrete centerline segments $\Gamma_j$ via different colors. \textbf{b} Visualization for the disjoint tubular neighbourhood regions $\cT_j$. Note that for better visualization we slightly enlarge the width of each tubular neighbourhood region.  \textbf{c} Visualization for the planar components $\gamma_j$ of the smooth paths $\cG_j$, each of which  fits to the discrete centerline segments. \textbf{c} The computed curvature of each fitting path. The red arrows indicate the respective parameterization directions of the physical projection curves $\gamma_j$ of the fitting minimal paths.}
\label{fig_CurvaturePrior}
\end{figure}

\begin{figure}[htbp]
\centering
\includegraphics[width=0.9\linewidth]{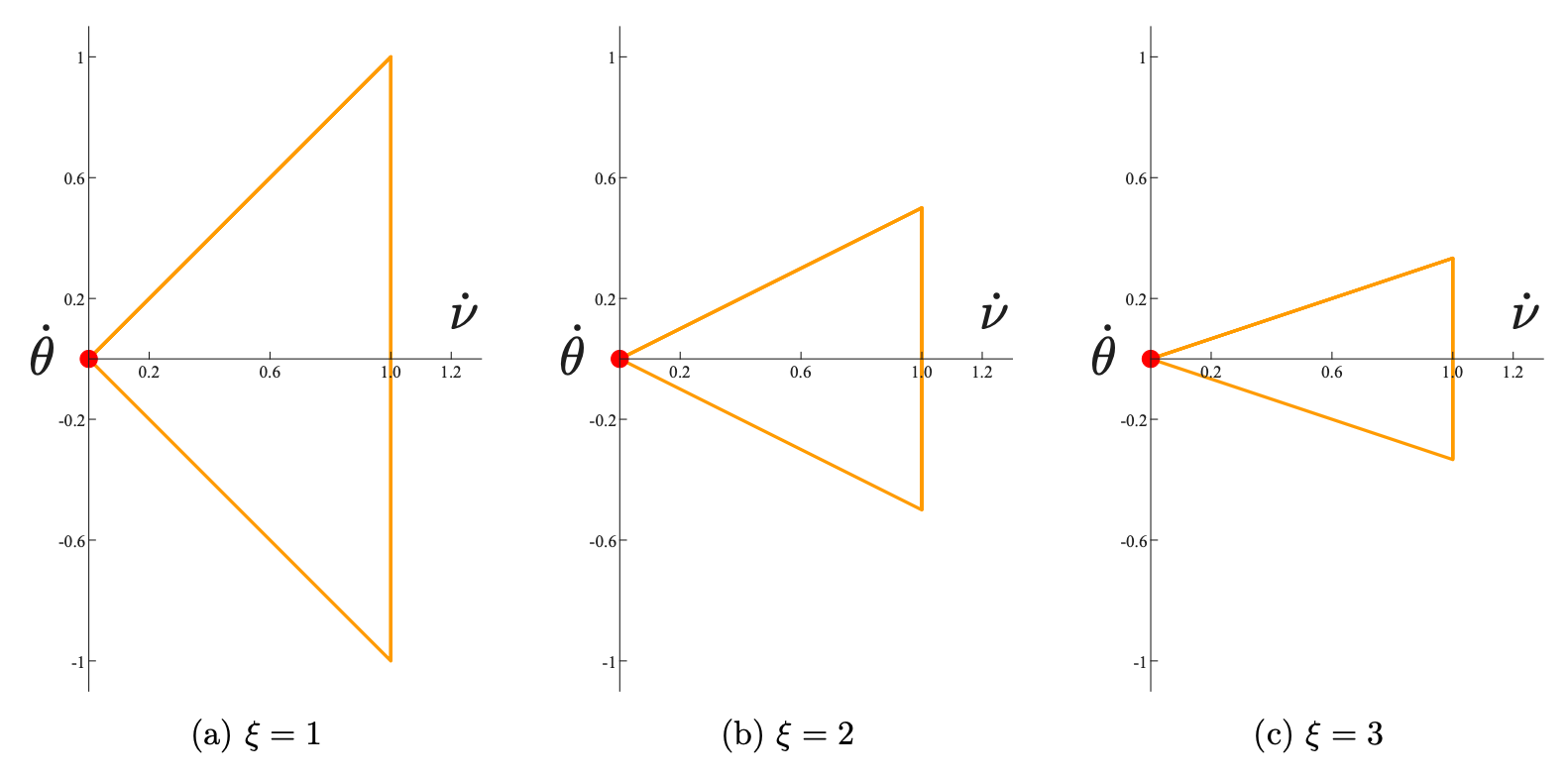}
\caption{Illustration for the sliced control sets $\tilde{\kB}^{\rm D}(\fp)$ of the Dubins model with different values of the parameter $\xi$. The red dots are the origins $(0,0)$ of the sliced control sets. (\textbf{a}) - (\textbf{c}) The sliced control sets with respect to $\xi=1$, $\xi=2$ and $\xi=3$, respectively.}
\label{fig_DubinsControlSet}
\end{figure}

\FloatBarrier





\bibliographystyle{plain}
\bibliography{minimalPaths}

\end{document}